\documentclass[11pt, peerreview]{IEEEtran}

\usepackage{amsmath}
\usepackage{amsfonts}
\usepackage{amssymb}
\usepackage{amsthm}
\usepackage{xr}
\theoremstyle{definition}

\theoremstyle{remark}
\newtheorem{remark}{Remark}

\usepackage{graphicx}
\usepackage{subcaption}
\usepackage[dvipsnames]{xcolor}

\usepackage{algorithm}
\usepackage{algpseudocode}

\usepackage{booktabs}
\usepackage{multirow}
\usepackage{multicol}

\usepackage{mathrsfs}
\usepackage{dsfont}
\usepackage{newpxmath}

\usepackage{setspace}
\newcommand{\Dtgt}{\mathfrak{D}_{\textrm{tgt}}}
\newcommand{\Dsrc}{\mathfrak{D}_{\textrm{src}}}
\newcommand{\Nsrc}{N_{\textrm{src}}}

\newcommand{\Qproc}{Q^w}
\newcommand{\Rmeas}{Q^\eta}
\newcommand{\Dtrain}{\mathfrak D_{\rm src}}
\newcommand{\Dtest}{\mathfrak D_{\rm tgt}}

\newcommand{\Real}{\mathbb{R}}

\newcommand{\Ncal}{\mathcal{N}}

\newcommand{\PP}{\boldsymbol{P}}

\newcommand{\tikzcustomremake}[2]{	\tikzset{external/remake next}}
\makeatletter
\renewcommand*\env@matrix[1][*\c@MaxMatrixCols c]{	\hskip -\arraycolsep
	\let\@ifnextchar\new@ifnextchar
	\array{#1}}
\makeatother

\newcommand{\fig}[1]{Fig.~\ref{fig:#1}}
\newcommand{\eqnlabel}[1]{\label{eq:#1}}
\newcommand{\eqn}[1]{(\ref{eq:#1})}

\renewcommand{\Real}{\mathbb{R}}

\newcommand{\ec}{\mathcal{E}}
\newcommand{\dc}{\mathcal{D}}
\newcommand{\projop}{\mathcal{P}}

\newcommand{\Pres}{\mathsf{P}}
\newcommand{\Enth}{\mathsf{h}}
\newcommand{\Temp}{\mathsf{\theta}}

\newcommand{\matvar}[1] {{#1}}

\newcommand{\x}{\matvar{x}}

\newcommand{\Gaussian}{\mathcal{N}}

\newcommand{\Fjac}{\matvar{F}}
\newcommand{\Hjac}{\matvar{H}}

\newcommand{\vgrad}{\vv\nabla}

\newtheorem{lemma}{Lemma}
\newtheorem{theorem}{Theorem}

\theoremstyle{definition}

\usepackage{newpxmath}

\title{Meta-Learning for  Physically-Constrained Neural System Identification}
\author{Ankush Chakrabarty$^*$$^\dag$\thanks{$^*$Corresponding author. Email: \texttt{achakrabarty@ieee.org}.}, Gordon Wichern,  Vedang M. Deshpande, Abraham P. Vinod, \\ Karl Berntorp, and Christopher R. Laughman
\thanks{$^\dag$All authors are affiliated with Mitsubishi Electric Research Laboratories, 201 Broadway, 8th Floor, Cambridge, MA 02139, USA.}}

\begin{document}
\maketitle
\begin{abstract}
We present a gradient-based meta-learning framework for rapid adaptation of neural state-space models (NSSMs) for black-box system identification. When applicable, we also  incorporate domain-specific physical constraints to improve the accuracy of the NSSM. The major benefit of our approach is that instead of relying solely on data from a single target system, our framework utilizes data from a diverse set of source systems, enabling learning from limited target data, as well as with few online training iterations. Through benchmark examples, we demonstrate the potential of our approach, study the effect of fine-tuning subnetworks rather than full fine-tuning,  and report real-world case studies to illustrate the practical application and generalizability of the approach to practical problems with physical-constraints. Specifically, we show that the meta-learned models result in improved downstream performance in model-based state estimation in indoor localization and energy systems.
\end{abstract}


\section{Introduction}
Data-driven system identification is often a necessary step for model-based design of control systems. While many data-driven modeling frameworks have been demonstrated to be effective, the class of models that contain a state-space description at their core have typically been easiest to integrate with model-based control and estimation algorithms, e.g., model predictive control or Kalman filtering. 
Early implementations of neural state-space models (SSMs) employed shallow recurrent layers and were dependent on linearization to obtain linear representations~\cite{zamarreno1998state} or linear-parameter-varying system representations~\cite{bao2020identification}. 
Recent advancements in deep neural networks have enabled introducing SSMs into the neural architecture explicitly without post-hoc operations~\cite{forgione2022learning}. Therefore the SSM description can be learned directly during training; see~\cite{legaard2021constructing} for a recent survey. Unmodeled dynamics can be represented using neural SSMs after constructing a physics-informed prior model~\cite{forgione2020model,forgione2021dynonet}, and additional control-oriented structure can be embedded during training~\cite{skomski2021automating}. 
Another interesting direction of research has led to the development of autoencoder-based SSMs, where the neural architecture comprises an encoder that transforms the ambient state-space to a (usually high-dimensional) latent space, a decoder that inverse-transforms a latent state to the corresponding ambient state, and a linear SSM in the latent space that satisfactorily approximates the system's underlying dynamics~\cite{masti2021learning,iacob2021deep,bertalan2019learning}. 
Even without the decoder, deep encoder networks have proven useful for neural state-space modeling~\cite{beintema2021nonlinear}. 
An argument for the effectiveness of autoencoder-based approaches is based on Koopman operator theory~\cite{koopman1932dynamical}, which posits that a nonlinear system (under some mild assumptions) can be lifted to an infinite-dimensional latent space where the state-transition is linear; an autoencoder allows a finite-dimensional, therefore tractable, approximation of the Koopman transformations~\cite{lusch2018deep}.

Almost without exception, neural SSMs have been constructed using data from a single \textit{target} system under consideration. In practice though, modeling experts often have access to data for a range of similar (not necessarily identical) \textit{source} systems. These source systems usually contain information that could speed up neural SSM training for a target system, as long as the target system is similar to some of the source systems. Herein, we propose the use of  gradient-based meta-learning to leverage data from similar systems to rapidly adapt a neural SSM to a new target system with (i) few data points from the target system, and (ii) with very few online training iterations. To the best of our knowledge, we are the first to propose gradient-based meta-learning for physically-constrained system identification of black-box dynamics.

Previous work that bears some resemblance to this idea include: reduction of model estimation error by using data from linear systems within a prescribed ball~\cite{9867413} and in-context learning with attention-based transformer networks that can represent SSMs for classes of dynamical systems, rather than single instances of that system class~\cite{forgione2023system, piga2024adaptation}. In~\cite{9867413}, the key difference with our approach is that the underlying class of systems considered are linear, which we greatly expand upon in this paper. The key differences with the line of work using transformer-based architectures is more nuanced. First, as argued in~\cite{bemporad2024linear}, many dynamical systems can be identified with shallow or slightly deep networks, and rarely require massive very large sizes, such as foundation models and transformer networks that are often extremely deep and contain millions (sometimes, billions) of parameters. Such massive transformer networks are beneficial for generalization and large-scale learning, but will  be difficult to deploy in control or estimation applications requiring online adaptation, fast Jacobian calculations for linearization, or implementation on-chip. Therefore, it is necessary to devise fast adaptation mechanisms for small-sized networks, which is the objective of this work. Second, gradient-based meta-learning does not require massive pre-training datasets, making it more amenable to applications where a large number of similar systems' data is not available. Conversely, we will show empirical evidence that our proposed methodology can train from small training sets and adapt with very few online iterations. Outside of system identification,  meta-learning has been proposed for optimization~\cite{zhan2022calibrating,chakrabarty2022meta}, adaptive control~\cite{DBLP:journals/corr/abs-2103-04490}, and receding horizon control~\cite{arcari2020meta,muthirayan2020meta,10711717}.

The main \textit{contribution} of this work is to propose a meta-learning approach for neural state-space modeling using data obtained from a range of similar systems. To reiterate, in meta-learning, a neural network is trained on a variety of similar source system models so that it can make accurate predictions for
a given target system model, with only a few data points from the target system and a small number of gradient-based adaptation steps. Concretely, we learn neural SSM weights from a variety of source systems' data. Consequently, even with a small amount of the target system's data, the SSM can be quickly adapted online to obtain a set of encoder weights and state/output matrices representing the target system dynamics. This meta-learned neural SSM is shown via a numerical example to result in higher predictive accuracy than a neural SSM trained solely using target system data, or even a neural SSM trained on the entire source plus target data. 

We employ model agnostic meta-learning (MAML) algorithms, which are trained by solving a bi-level optimization problem~\cite{finn2017model}, wherein the outer loop extracts task-independent features across a set
of source tasks, and the inner loop adapts to a specific model with a few iterations and limited data. Since the bi-level training paradigm can often lead to numerical instabilities~\cite{antoniou2019train}, a recent variant of MAML, referred to as almost-no-inner-loop (ANIL), slices the network into a base-learner and a meta-learner, and dispenses with (or significantly cuts) inner-loop updates, to improve meta-training performance~\cite{raghu2019rapid}. Additional improvements have been made to the base MAML algorithm to reduce memory- and sample-complexity such as implicit MAML~\cite{rajeswaran2019meta}, Reptile~\cite{nichol2018reptile}, and a perturbation-based first-order MAML~\cite{chayti2024new}, but we focus on the MAML and ANIL algorithms as the proposed approach can easily be adapted to these variants. Additional contributions include: incorporating physical constraints on the unconstrained neural SSM for two real-world case-studies, benchmarking the proposed approach against existing neural system identification methods, and investigating the effectiveness of full MAML adaptation versus ANIL adaptation of a subset of neural SSM layers.

The rest of the paper is organized as follows. In Section~\ref{sec:neural_sys_id}, we introduce the neural state-space model and formally state the system identification problem when similar systems' data is available. In Section~\ref{sec:maml_anil}, we present the meta-learning approach for physics-constrained system identification. In Section~\ref{sec:example}, we illustrate the proposed approach using two numerical examples, one for benchmarking against previous neural system identification methods, and one for demonstrating the effectiveness of MAML and ANIL. In Section~\ref{sec:case_studies}, we present two real-world case studies to demonstrate the practical application of the proposed approach. The appendices contain additional explanatory figures.

\section{Neural System Identification}
\label{sec:neural_sys_id}
\begin{figure}[!ht]
\centering
\includegraphics[width=0.75\textwidth]{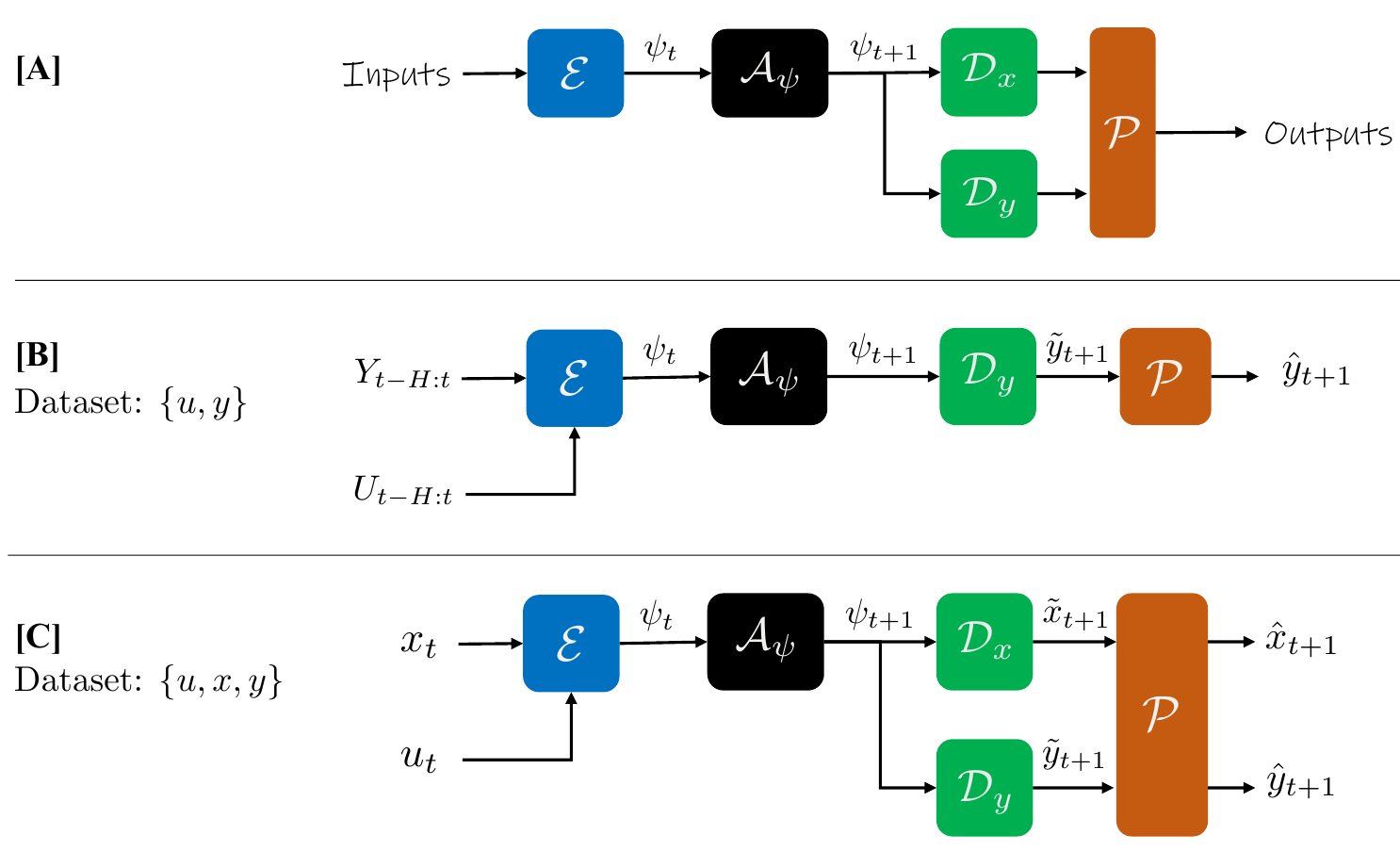}
\caption{Neural state-space models: general abstraction, and specific models for different  target dataset composition.}
\label{fig:nssm_general}
\end{figure}
\subsection{Motivation}
We assume that the \underline{target} dynamical system under consideration has the form
\begin{align}
\label{eq:nonlinear_system}
x_{t+1} &= f(x_t,u_t|\theta^\star), \;\;
y_t = h(x_t,u_t|\theta^\star),
\end{align}
where $x_t\in\mathbb{R}^{n_x}$ is the state, $u_t\in\mathbb{R}^{n_u}$ is the input, $y_t\in\mathbb{R}^{n_y}$ is the output, and $\theta^\star\in\mathbb{R}^{n_\theta}$ is the parameter vector. We do not assume that the modeling functions $f$ and $h$ are known, only that they satisfy basic continuity and differentiability conditions, so that the system is well-posed, and solutions to the difference equations exist and are unique. 

We assume that the unknown parameter vector $\theta^\star$ belongs to a compact set $\Theta\subset\mathbb{R}^{n_\theta}$, and that the set $\Theta$ is known at design time. This is not a restrictive assumption, as most physical systems have models with parameters that can be, for instance, bounded within a certain range. 

We assume the availability of a \emph{target dataset} $\Dtgt$ from the system~\eqref{eq:nonlinear_system}, with limited dataset size. In the classical system identification setting, the target dataset contains input-output $\{u_t, y_t\}_{k=0}^{T^\star}$ data, with small $T^\star\in\mathbb N$. Alternatively, one may have access to a specific state representation predefined in a computer simulation environment: for example, in a digital twin of the target system under consideration. In such cases, the state representation $x_t$ may be available, and the target dataset $\Dtgt$ may contain state-input $\{x_t, u_t, y_t\}_{k=0}^{T^\star}$ data.  Although the modeling functions $f$ and $h$ may be accessible in this setting, advanced simulation software often comprise modules of such high model complexity that attempting model-based design with these modules is often impractical, and sometimes intractable: for instance, if the model components cannot be represented using closed-form expressions.

Our broad objective is to leverage the limited dataset $\Dtgt$ to learn a neural predictive model that predicts the evolution of the system~\eqref{eq:nonlinear_system} accurately using a history of measured variables, while enforcing domain-specific or `physics' constraints on the system. The physics constraints are typically derived from first principles, and are used to ensure that the learned model is physically consistent with the target system. We do not claim to be able to handle general nonlinear constraints, and limit ourselves to specific classes of constraints that arise in systems that we present later in Section~\ref{sec:physics-ai}. These constraints can be explicitly modeled using specific layers and activations in a neural network architectures.

\subsection{Neural State-Space Modeling}

To fulfill our objective, we propose using a neural state-space model (NSSM) architecture that can be represented in a general form; see Figure~\ref{fig:nssm_general}. Inputs to the NSSM are passed through encoding layers $\mathcal E(\cdot)$ to transform the inputs into a latent state $\psi_t$, which is subsequently passed through a state-transition operator $\mathcal A_{\psi}(\cdot)$ to obtain the updated latent state $\psi_{t+1}$. This updated latent state is decoded by the layers $\mathcal D_x(\cdot)$ and $\mathcal D_y(\cdot)$ to obtain intermediate variables $\tilde x_{t+1}$ and $\tilde y_{t+1}$, from which the predicted output $\hat y_{t+1}$ and updated state $\hat x_{t+1}$ is computed by enforcing the physical constraints using the constraint operator $\mathcal P(\cdot)$.
The following cases illustrate how the NSSM architecture changes to accomodate different types of data available in the target dataset $\Dtgt$; without loss of generality, we show the operations of the NSSM for $t=0$.

\subsubsection*{Case I, $\Dtgt = \{u, y\}$}
If the target dataset $\Dtgt$ contains input-output data $\{u_t, y_t\}_{t=0}^{T^\star}$, the NSSM architecture can be modified as shown in Figure~\ref{fig:nssm_general}[B]. In this case, the encoding layers $\mathcal E(\cdot)$ transform a history of inputs $U_{t-H:t}$ and outputs $Y_{t-H:t}$ into a latent state $\psi_t$, where $H$ is a user-defined window length of past data. The latent $\psi_t$ is updated using the state-transition operator $\mathcal A_{\psi}(\cdot)$, and the updated latent state $\psi_{t+1}$ is decoded by the layers $\mathcal D_y(\cdot)$ to obtain an intermediate variable $\tilde y_{t+1}$. The predicted output $\hat y_{t+1}$ is computed by enforcing the physical constraints using the constraint-enforcing operator $\mathcal P(\cdot)$. The main operations are summarized by the equations:
\begin{subequations}
\label{eq:nssm_case1}
\begin{align}
\label{eq:nssm-case1-a}
\psi_t &= \mathcal E(U_{t-H:t}, Y_{t-H:t}),\quad
\psi_{t+1} = \mathcal A_{\psi}(\psi_t), \\
\label{eq:nssm-case1-b}
\tilde y_{t+1} &= \mathcal D_y(\psi_{t+1}), \quad
\hat y_{t+1} = \mathcal P(\tilde y_{t+1}).
\end{align}
\end{subequations}

Training the neural SSM involves computing weights of $\mathcal E$, $\mathcal A$, $\mathcal D_y$, and  $\mathcal P$. This is performed by minimizing a multi-step prediction loss as follows. For an input history $Y_{t-H:t}$ and $U_{t-H:t}$, the latent encoding $\psi_t$ is computed using~\eqref{eq:nssm-case1-a}, after which, for a user-defined prediction horizon of $H_{p}$, we recursively compute $\psi_{t+1},\cdots, \psi_{t+H_p}$ with the state-transition operator $\mathcal A_{\psi}$. Subsequently, we can compute
$\hat Y_{t:t+H_p-1} = \{\hat y_{t}, \hat y_{t+1}, \cdots, \hat y_{t+H_p-1}\}$ using~\eqref{eq:nssm-case1-b}. As the classical NSSM training is done offline, the output $y_t$ is available, with which one can construct $Y_{t:t+H_p-1}$. Then, the multi-step predictions can be evaluated via a mean-squared-error (MSE) loss function
\begin{equation}\label{eq:nssm-loss-fn-1}
\mathsf{L}^{uy}_{\rm SSM} = \mathsf{MSE}_{t:t+H_p-1}\big(y_{t+1}, \hat y_{t+1}\big)
\end{equation}
that can be minimized using batched data and stochastic gradient descent methods.

\subsubsection*{Case II, $\Dtgt = \{x, u, y\}$}
If the target dataset $\Dtgt$ contains state-input-output data $\{x_t, u_t, y_t\}_{t=0}^{T^\star}$, the NSSM architecture is modified as shown in Figure~\ref{fig:nssm_general}[C]. In this case, the encoding layers $\mathcal E(\cdot)$ transform the state $x_t$ and input $u_t$ into a latent state $\psi_t$, which is passed through the state-transition operator $\mathcal A_{\psi}$ to obtain the updated latent state $\psi_{t+1}$, from which decoders $\mathcal D_x$ and $\mathcal D_y$ compute intermediate variables $\tilde x_{t+1}$ and $\tilde y_{t+1}$, respectively. The predicted state $\hat x_{t+1}$ and output $\hat y_{t+1}$ are computed by enforcing the physical constraints using the constraint-enforcing operator $\mathcal P$. The main operations are summarized by the equations:
\begin{subequations}
\label{eq:nssm_case2}
\begin{align}
\label{eq:nssm-case2-a}
\psi_t &= \mathcal E(x_t, u_t),\;
\psi_{t+1} = \mathcal A_{\psi}(\psi_t), \\ \tilde x_{t+1} &= \mathcal D_x(\psi_{t+1}),\quad
\tilde y_{t+1} = \mathcal D_y(\psi_{t+1}),\\
\label{eq:nssm-case2-c}
\hat x_{t+1}, \hat y_{t+1} &= \mathcal P(\tilde x_{t+1}, \tilde y_{t+1}).
\end{align}
\end{subequations}

Training the neural SSM involves computing weights of $\mathcal E$, $\mathcal A$, $\mathcal D_x$, and $\mathcal D_y$. Since the network~\eqref{eq:nssm_case2} is reminiscent of neural Koopman operators~\cite{mauroy2020koopman}, we use the same loss functions as proposed in that literature. Namely, we minimize a loss function
\begin{equation}\label{eq:nssm-loss-fn-2}
\mathsf{L}^{uxy}_{\rm SSM} = \mathsf{L}_{\rm recon} + \mathsf{L}_{{\rm pred}, x} + \mathsf{L}_{{\rm pred}, y}.
\end{equation}
The components of this loss function are:
$$\mathsf{L}_{\rm recon} = \mathsf{MSE}_{t-H:t}\big(x_t \, , \mathcal P \circ\mathcal D_x\circ\mathcal{E}(x_t, u_t)\big),$$  $$\mathsf{L}_{\mathrm{pred},x} = \mathsf{MSE}_{t:t+H_p-1} \big(x_{t+1}, \hat x_{t+1} \big),$$ $$\mathsf{L}_{\mathrm{pred},y} = \mathsf{MSE}_{t:t+H_p-1}(y_{t+1}, \hat y_{t+1}),$$
where the reconstruction loss ensures consistency between the lifting encoder and the de-lifting decoder, and the prediction losses ensure good predictive performance. Note that $\mathsf{L}_{\rm recon}$ has an abuse of notation, as $\mathcal P \circ\mathcal D_x$ actually indicates only the state reconstruction component.

\begin{remark}
Note that~\eqref{eq:nssm-loss-fn-2} can be extended to multi-step prediction by recursively generating a sequence of latents $\psi_{k+1:k+N_S}$ for $N_S$ steps starting from $\psi_k$ using the state-transition operator $\mathcal A_{\psi}$ and inputs $u_{k:k+N_S-1}$. Decoding these latents would result in multi-step state and output prediction, with which MSE losses could be computed. 
\end{remark}

\subsection{Extended Kalman filtering}

We demonstrate the practical application of our proposed approach through real world examples using a meta-learned predictive model for subsequent applications, specifically, model-based estimation. We adopt an extended Kalman filtering (EKF) framework for state estimation because of its applicability to a wide range of nonlinear systems \cite{crassidis_optimal_2011}. 

In this section, we will briefly review an EKF for a general dynamic system expressed in the following form
\begin{subequations}
\label{eq:nonlinear_system_ekf}
\begin{align}
x_{t+1} &= \hat f(x_t,u_t) + w_t,\\
y_t &= \hat h(x_t,u_t) + \eta_t,
\end{align}
\end{subequations}
which is amenable to filtering formalism and follows directly from~\eqn{nonlinear_system}. Note that $\hat f(\cdot)$ and $\hat h(\cdot)$ here represent mathematical prediction models of the target system~\eqn{nonlinear_system}. 
The deviations of the predictive models \eqn{nonlinear_system_ekf} from the real system is captured by the random vectors $w_t$ and $\eta_t$ which are assumed to be zero-mean Gaussian uncertainties with covariances $\Qproc_t$ and $\Rmeas_t$, respectively. 

We use the following notations for EKF estimates: ($x_t^{-}$,$P_t^{-}$) and ($x_t^{+}$,$P_t^{+}$) respectively denote the prior (measurements assimilated up to time $t-1$) and posterior (measurements assimilated up to time $t$) mean-covariance pairs of the state vector at time $t$, i.e.,
\begin{align}
    x_{t|t-1} \sim \Gaussian(x_t^{-}, P_t^{-}) \text{ and } x_{t|t} \sim \Gaussian(x_t^{+}, P_t^{+}),
\end{align}
where $\Gaussian(\cdot)$ denotes the multivariate normal distribution.

For a given ($x_0^{-}, P_0^{-}$) and neural predictive model~\eqn{nonlinear_system_ekf}, the EKF equations are given by the time updates:
\begin{subequations}\eqnlabel{deep_cekf_all}
\begin{align}\eqnlabel{deep_cekf}
\PP_{k+1}^{-} &= \Fjac_t P_t^{+} \Fjac_t^\top + \Qproc_t, \\
\x^{-}_{k+1} &= f(x_t^+, u_t),
\end{align}
and the measurement updates:
\begin{align}
    K_t &= (P_t^{-}{H}_t^\top)({H}_t P_t^{-}{H}_t^\top + \Rmeas_t)^{-1}, \\
    P_t^{+}  &= (I - K_t{H}_t)P_t^{-}, \\
    x_t^{+} &= x_t^{-} + K_t\big({y}_t -  h(x_t^-, u_t)) \big), \eqnlabel{ekf_mean_up}
\end{align} 
\end{subequations}
where $\Fjac_t$ and $\Hjac_t$ are Jacobian matrices defined as
\begin{align}\eqnlabel{jac_defs}
    \Fjac_t & \triangleq \frac{\partial}{\partial x} f( {x_t^{+} }, u_t), \quad
    \Hjac_t \triangleq \frac{\partial}{\partial x} h(x_t^-, u_t).
\end{align}

In this work, we propose the use of meta-learned physics-constrained NSSMs described in \fig{nssm_general} as the prediction models in \eqn{nonlinear_system_ekf} for implementing the EKF \eqn{deep_cekf}. Note that computing the Jacobian matrices in ~\eqn{jac_defs} requires taking gradients of the neural operators, i.e., differentiating the outputs of the deep neural networks with respect to their inputs. To this end, we utilize auto-differentiation (AD) capabilities available in standard deep learning libraries such as \texttt{PyTorch}.

\subsection{Problem Statement}
In practice, we do not always have access to a large amount of data from the target system. Therefore, training NSSMs with limited data results in overfitting and poor generalization, i.e., inaccurate target model predictions. Fortunately, it is often the case that data from similar systems to the target system has been collected in the past. For instance, previously deployed engineering systems similar to the target application may have been monitored, and data generated from these systems can supplement the limited target dataset $\Dtgt$ to improve the generalization of the learned model. We refer to this additional dataset as the \underline{source} dataset $\Dsrc$. The source dataset $\Dsrc$ contains input-output or state-input-output data from a system that is similar to the target system, but not identical. The source dataset $\Dsrc$ is assumed to be available at design time, and to have the same signal composition as the target dataset $\Dtgt$.

More formally, to reduce the data required from the target $\theta^\star$-system, we leverage simulation models to obtain data from $\Nsrc\in\mathbb N$ dynamical systems that are similar to the target system. These \emph{source systems} adhere to the form~\eqref{eq:nonlinear_system} and are parameterized by  $\theta^\ell\in\Theta\setminus\theta^\star$ for $\ell=1,2,\ldots, \Nsrc$. By collecting data from each of the $\Nsrc$ source systems, for instance during field experiments or via digital-twin simulations, we have access to a source dataset, which can be written as
\begin{equation}\label{eq:Dtrain}
\begin{split}
\Dtrain&\triangleq\left\{u_{0:T_\ell},  y_{0:T_\ell}|\theta^\ell\right\}_{\ell=0}^{N_{\rm src}}  \\ \text{or}\;\;
\Dtrain&\triangleq\left\{x_{0:T_\ell}, u_{0:T_\ell},  y_{0:T_\ell}|\theta^\ell\right\}_{\ell=0}^{N_{\rm src}},
\end{split}
\end{equation}
obtained over a time-span from zero to $T_\ell\in\mathbb N$ for some (not necessarily identical) initial condition and excitation inputs $U_{0:T_\ell}$. Often, the target system state data is collected with expensive instrumentation over a small time span (not necessarily real-time), so $T^\star \ll \inf_\ell T_\ell$. 

By leveraging \textit{meta-learning}, we can  utilize the source data $\Dtrain$ to learn NSSMs that can represent models with different $\theta\in\Theta$ parameters and rapidly adapt (i.e., with few online training updates) to the target system dynamics with unknown parameters $\theta^\star$ despite limited target data. This adaptive identification approach also has the added practical benefit that no explicit estimation of the parameter $\theta^\star$ is required, which would likely require more data. Our objective is to design a rapid adaptation mechanism using meta-learning that can yield an accurate NSSM of the target system by learning from similar source systems. The adapted model can then be employed for downstream model-based control or model-based estimation tasks, as we will show in Section~\ref{sec:case_studies}.

\section{Model-Agnostic Meta-Learning (MAML) for  Identification}
\label{sec:maml_anil}
MAML~\cite{finn2017model} is one one of the most well-known and widely used meta-learning (i.e., learning to learn) algorithms. The goal of MAML is to learn a reusable set of network weights that can be quickly fine-tuned at inference time based on a small amount of adaptation data. MAML achieves this using a nested training scheme where an inner loop fine-tunes a common set of initial model parameters based on a small amount of task-specific adaptation data, and an outer loop that updates the initial set of model parameters across a mini-batch of different tasks. This procedure promotes learning of model parameters that can quickly adapt to new tasks. An overview of our approach applying MAML for trajectory prediction is shown in Supplementary Fig.~S1. 

Perhaps the most non-trivial aspect of implementing meta-learning algorithms such as MAML is data partitioning. Formally, we are given a dataset of source systems $\Dtrain$ defined in~\eqref{eq:Dtrain}. We can partition the data for the $\ell$-th system into a \emph{context} dataset $\mathcal{C}^{\ell}$ for inner-loop updates, and a \emph{target} set $\mathcal{T}^{\ell}$ to evaluate the loss function for the model parameters adapted in the inner-loop. 
If $\omega$ denote the set of trainable NSSM parameters\footnote{Note that care must be taken to ensure that the parameters that are meta-trained are indeed trainable. For example, parameters enforcing physics constraint layers $\mathcal P$ may need to be fixed, as we shall show in the case studies.}, then the inner-loop MAML update is
\begin{equation}\label{eq:maml_inner}
    \omega_{m}^{\ell}=\omega_{m-1}^{\ell}-\beta_{\sf in}\nabla_{\omega_{m-1}^{\ell}}\mathsf{L}_{\rm SSM}\big(\mathcal{C}^\ell;\omega_{m-1}^{\ell}\big)
\end{equation}
where $m$ denotes the iteration of inner-loop updates, $\beta_{\sf in}$ the inner-loop learning rate, and $\mathsf{L}_{\rm SSM}\big(\mathcal{C}^\ell;\omega_{m-1}^{\ell}\big)$ is the loss function from~\eqref{eq:nssm-loss-fn-1} or~\eqref{eq:nssm-loss-fn-2}, evaluated on the context set with the neural weights computed after $m-1$ inner-loop updates. The inner-loop updates are performed individually using the context set for each trajectory in the batch, while the target sets for all trajectories in the batch are used in the outer-loop. The outer-loop optimization step is typically written as
\begin{equation}\label{eq:maml_outer}
\omega=\omega-\beta_{\sf out}\nabla_{\omega}\sum_{b=1}^{B}\mathsf{L}_{\rm SSM}\big(\mathcal{T}^b;\omega_{m}^{b}\big)
\end{equation}
where $B$ is the number of tasks (i.e., trajectories) in a training mini-batch, $\beta_{\sf out}$ is the outer-loop learning rate, and $\mathsf{L}_{\rm SSM}\big(\mathcal{T}^b;\omega_{m}^{b}\big)$ is the loss function on the target set $\mathcal{T}^{b}$ after $m$ inner-loop iterations~\eqref{eq:maml_inner} have been completed. By updating the model parameters in the outer-loop as in~\eqref{eq:maml_outer} across $B$ tasks, we obtain a parameter set $\omega$ that can be quickly adapted at inference time. 
At inference-time we only perform $m$ inner-loop updates~\eqref{eq:maml_inner} using the target system data $\mathcal C^\star=\Dtgt$ and evaluate NSSM performance using updated parameters $\omega_m^{\star}$.

When partitioning trajectory data into context sets used for inner-loop fine-tuning and target sets used to evaluate the fine-tuning, one may use different approaches for training and inference. At inference time, we will typically use the first several points of the observed trajectory as the context set, and any subsequent points as the target set to simulate the real-world scenario, where we first fine-tune our meta learned model and then use it. At training time, we randomly sample consecutive points from anywhere in the trajectory as the context set, and consecutive points from anywhere in trajectory (not necessarily after the context set) as the target set. We do this to ensure that the learned model does not always expect the context set to depend on initial conditions; for instance, when the target data is taken from steady-state conditions.

The nested training scheme in MAML can lead to a very difficult optimization problem~\cite{antoniou2019train}, and updating all model parameters in the inner-loop may be unnecessary~\cite{raghu2019rapid}. For this reason, the almost no inner-loop (ANIL) algorithm was proposed in~\cite{raghu2019rapid}, which updates only the parameters in the last layer of the network in the inner-loop~\eqref{eq:maml_inner}, while parameters for all layers are updated in the outer-loop~\eqref{eq:maml_outer}. The intuition being that the outer-loop update promotes extraction of low-level features that are reusable across tasks, and the inner-loop promotes rapid learning of a final task specific layer. 

In the few-shot image classification problem commonly studied in meta-learning research, the last layer is referred to as a ``classifier'' and the earlier layers that are not updated in the inner-loop as the ``feature extractor.'' However, for the neural SSM architecture studied in this work, it is unclear whether the encoder $\mathcal E$ or the state-space model parameters $\mathcal A_{\psi}$ should be common across tasks or fine-tuned in the inner-loop, so we explore both possibilities. Formally, if $\omega=(\omega_1,...,\omega_L)$, where $\omega_l$ are the network parameters for layer $l$ and $L$ is the total number of layers in the network, we then define $\omega_{\sf in} \subseteq \omega$ as the subset of parameters updated in the inner-loop, and the inner-loop update from~\eqref{eq:maml_inner} becomes
\begin{equation}\label{eq:anil_inner}
    (\omega_l)_{m}^i\!=\!
    \begin{cases}
        \omega_l, & \!\!\!\!\!\omega_l \!\notin\! \omega_{\sf in} \\
        (\omega_l)_{m-1}^{i}\!-\!\beta_{\sf in}\nabla\mathsf{L}_{\rm SSM}\big(\mathcal{C}_{Y^i};(\omega_l)_{m-1}^{i}\big), & \!\!\!\!\!\omega_l\!\in\!\omega_{\sf in}
    \end{cases}
\end{equation}
The outer-loop of ANIL remains unchanged from MAML, and we note that ANIL and MAML are identical when $\omega_{\sf in}=\omega$. We summarize the meta-training regime using source systems' data via ANIL and MAML in Supplementary Algorithm~S1, and explain how to update the weights for the target system in Supplementary Algorithm~S2.

\section{Physically-Constrained NSSMs}
\label{sec:physics-ai}

In the extension of the proposed MAML work, we explicitly integrate physics constraints into the NSSM architecture to ensure that the predicted system trajectories are physically realizable. We consider two classes of constraints that may stem from the underlying physical laws governing the system's behavior. First, we propose an approach to constrain model predictions within a polytope.
Next, we introduce an approach to guarantee that the model predictions analytically satisfy the curl-free property from vector calculus. As demonstrated in Section \ref{sec:case_studies} through real-world applications, incorporating physics constraints into the NSSM architecture not only enhances model prediction accuracy but also improves downstream estimation performance.

\subsection{Polytopic constraints}
\label{sub:physics_VCS}

In certain applications, we know that the output of the NSSM $\hat x_t$
must satisfy constraints that can be encoded as 
$\hat x_t\in\mathcal G$, 
where $\mathcal G$ is a polytope 
\begin{equation}\label{eq:polyhedron}
\mathcal G := \left\{x\in\mathbb X: G x\le 0\right\},
\end{equation}
for some known matrix $G$.  
From the Minkowski-Weyl Theorem, the polytope $\mathcal G$ admits an equivalent description using  rays~\cite{fukuda1995double}\footnote{To be accurate,  Minkowski-Weyl states that the feasible polytope of an inhomogeneous linear inequality $Gx\le g$ can be equivalently described by a conic combination of rays and a convex combination of vertices. When $g=0$, only the conic combination remains.}. That is, for some finite set of rays $\mathcal R:=\begin{bmatrix}
    \mathcal R_1 & \cdots & \mathcal R_{r}
\end{bmatrix}\in \mathbb R^{n_x\times r}$ we can write $\mathcal G = \mathsf{cone}(\mathcal R)$, where $\mathsf{cone}(\mathcal R):= \sum_{j=1}^{r} \mu_j \mathcal R_j$ for some $\mu_j\ge 0$. 
In order to embed the physics constraint into the NSSM, we propose the following structure for the constraint layer $\mathcal P$. For the clarity of discussion, we decompose $\mathcal P$ into two parts $\mathcal P_x$ and $\mathcal P_y$ which act separately on the outputs of decoders $\mathcal D_x$ and $\mathcal D_y$ respectively. While we set $\mathcal P_y$ to be an identity matrix, we enforce the following structure on $\mathcal P_x$. 
\begin{align}\eqnlabel{constr_structure_vccs}
    \mathcal P_x:= \mathcal R \mu, \text{ where } \mu := \mathsf{ReLU}\circ \mathsf{FC}_r \circ \mathcal D_x(\psi),
\end{align}
with $\mathsf{ReLU}(\cdot)$ denoting the rectified linear unit activation function, and $\mathsf{FC}_r$ denoting a fully connected layer with output dimension $r$.
That is, we pass the decoder output $\mathcal D_x(\psi)$ through a fully connected layer activated by a ReLU function to make the vector element-wise non-negative, and therefore fulfill the requirement to act as $\mu\!\ge\!0$ in the conic constraint. This fully connected layer also maps $\mathcal D_x(\psi)$ to a vector $\mu$ that is compatible in dimensions with the ray matrix $\mathcal R$. This is because we take the product of $\mathcal R$ and $\mu$, which by construction, results in a feasible predicted state $\bar x$. 
A general NSSM \eqn{nssm_case2} corresponding to \fig{nssm_general}[C] with the constraint structure \eqn{constr_structure_vccs} was learned by minimizing the loss \eqn{nssm-loss-fn-2} using Algorithm~\ref{alg:maml}, and the weights were fine-tuned using Algorithm~\ref{alg:maml_inference} for the target systems.

Alternatively, consider the case where the constrained NSSM~\eqn{nssm_case2} is used for state estimation. 
For the clarity of discussion, we explicitly write the NSSM-based model equations similar to \eqn{nonlinear_system_ekf}, as follows
\begin{subequations}\eqnlabel{vcc_model_for_ekf}
    \begin{align}
    x_{k+1} & = \projop_x \circ\dc_x\big(\mathcal{A}\circ\ec(x_t, u_t) \big) + w_t, \\
    y_t &= \dc_y \circ\ec(x_t, u_t) + \eta_t,
    \end{align} 
\end{subequations}
where $w_t$ and $\eta_t$ are assumed to be zero-mean Gaussian uncertainties with covariances $\Qproc_t$ and $\Rmeas_t$, respectively. 
The EKF \eqn{deep_cekf} can be readily implemented using \eqn{vcc_model_for_ekf} as the prediction models. The Jacobians $\Fjac_t$ and $\Hjac_t$ used in EKF are computed as follows
\begin{subequations}\eqnlabel{jac_defs_vcs}
    \begin{align}
    \Fjac_t & \triangleq \frac{\partial}{\partial x} \projop_x\circ\dc_x\circ\mathcal{A}\circ\ec( {x_t^{+}, u_t}), \\
    \Hjac_t & \triangleq \frac{\partial}{\partial x} \dc_y \circ\ec({x_t^{-}, u_t}).
\end{align} 
\end{subequations}

Despite designing the NSSM~\eqn{vcc_model_for_ekf} to satisfy the state constraints~\eqn{polyhedron}, the filter updates using~\eqn{deep_cekf} are not designed to explicitly satisfy the constraints, which could yield physics-inconsistent state estimates. Therefore, we deem it  necessary to explicitly incorporate state constraints during the filtering process as well~\cite{deshpande2022Constrained, dansimon2010PDFtrunc}. 

There are several approaches \cite{amor2022constrainedReview} such as density truncation \cite{dansimon2010PDFtrunc}, estimate projection \cite{deshpande2022Constrained} and pseudo or perfect measurements \cite{deshpande2023MCEKSifac}, available in the literature that can be adapted for enforcing constraint in state estimation algorithms. However, unlike previous works, here we leverage the inherent constraint-satisfying nature of the NSSMs in order to enforce the same constraints on state vectors estimated by the EKF. In particular, we propose the use of auto-encoder pair $(\ec \, , \projop_x \circ \dc_x)$ for projecting the filter estimate obtained in \eqn{ekf_mean_up} on the feasible set to enforce constraints. The constraint layer $\projop_x$ ensures that the decoded state vector satisfies the gradient constraint. Therefore, a constrained state estimate can be obtained by the following reconstruction using the auto-encoder
\begin{align} \eqnlabel{ae_based_proj}
    \x_t^{+} \gets \projop_x\circ\dc_x\circ\ec(\x_t^{+}, u_t). 
\end{align}
The constraint-projection operation in \eqn{ae_based_proj} is expected to computationally efficient as it involves only forward evaluation of neural networks. This is in contrast to other prevalent constraint enforcement approaches which typically requires solving a constrained quadratic program (QP) \cite{deshpande2022Constrained, dansimon2010PDFtrunc}. 
We show in 
Section~\ref{sub:num_VCS} that polytopic constraints are critical to incorporate physics while modeling vapor compression cycles via NSSMs to ensure unidirectional refrigerant flow.

\subsection{Vector-field constraints}
\label{sub:physics_mag}

We now consider a scenario where the output 
is a vector field that must meet vector field constraints, specifically the curl-free constraint. This type of constraint can arise from physical conservation laws that govern the system, as seen in the localization application discussed in Section~\ref{subsec:localization}. In this case, the output models a magnetic field that is curl-free due to the absence of internal currents~\cite{Wahlstroem2013}.
In such mechanical engineering applications, it is common for the state $x$ to contain position information in terms of spatial coordinates. That is, some components of the state vector are known to be $[p_X,p_Y]$ or $[p_X,p_Y,p_Z]$, which denote spatial coordinates in a 2- or 3-dimensional Cartesian coordinate system.
Therefore, the curl-free constraint requires that $\vgrad \times h = 0$, where $\vgrad$ is the vector differential operator,
$\vgrad = \begin{bmatrix}
    \tfrac{\partial}{\partial p_X} & \tfrac{\partial}{\partial p_Y} & \tfrac{\partial}{\partial p_Z}
\end{bmatrix}^\top$ for 3-dimensional coordinates and $\vgrad = \begin{bmatrix}
    \tfrac{\partial}{\partial p_X} & \tfrac{\partial}{\partial p_Y} 
\end{bmatrix}^\top$ for 2-dimensions.

We again consider the base learner shown in Fig.~\ref{fig:nssm_general}[C]. The key operations in this physics-constrained NSSM are as described in~\eqref{eq:nssm_case2}. We select the physics-constraint layer $\mathcal P$ to be a vector gradient operator acting on the output channels, namely
$$\hat x_{t+1}, \hat y_{t+1} = \tilde x_{t+1}, \vgrad \tilde y_{t+1}.$$
In other words, if we decompose $\mathcal P$ into two parts $\mathcal P_x$ and $\mathcal P_y$ which act separately on the outputs of decoders $\mathcal D_x$ and $\mathcal D_y$ respectively, then $\mathcal P_x$ corresponds to an identity transformation and $\mathcal P_y$ involves computing a gradient (i.e., $\mathcal P_y = \vgrad$). Next, we discuss how the output path of the autoencoder implicitly enforces the curl-free constraint on the measurement function $h$.

The addition of physics-informed vector-field constraints into deep neural networks has been described more generally in~\cite{Hendriks2020}, and we follow their method to enforce specific curl-free constraints into our NSSM, with some additional modifications to improve training performance. 
Clearly, such a constraint would have to be enforced on the neural network mapping from components of the state $x$ to the output $y$. To this end, we invoke the following result~\cite[pp. 444]{vectorCalculusBook}.
\begin{theorem}
\label{thm:1}
Let $h$ be a vector field of class $\mathcal C^1$ whose domain is a simply connected region $\mathcal R\subset \Real^3$. Then $h = \vgrad \varphi$ for some scalar-valued
function $\varphi$ of class $\mathcal C^2$ on $\mathcal R$ if and only if $\vgrad\; \times\; h = 0$ at all points of $\mathcal R$.
\end{theorem}
This fact implies that if we want to model a vector-field $h$ that is curl-free using a neural network, a reasonable strategy is to learn a scalar function $\varphi(x)$, whose gradient $\vgrad \varphi$ will yield the desired curl-free vector field.
Since modern deep-learning toolkits are equipped with capable automatic differentiation (AD) modules, computing $y$ from $\varphi$ is straightforward as long as the components of the state $x$ corresponding to Cartesian coordinates $[p_X, p_Y, p_Z]\subset x$ are known. If so, computing $\vgrad \varphi$ merely involves executing an automatic differentiation with respect to these components. This motivates the design of the output path (i.e., $\mathcal P_y = \vgrad$) in our proposed physics-constrained NSSM. To reiterate, the estimate of the output function $h$, guaranteed by Theorem~\ref{thm:1} to be curl-free, is 
\begin{equation}\label{eq:h_bar}
\bar h(x,u)=\vgrad\dc_y\circ\ec(x, u).
\end{equation}

In order to represent non-constant $h$, one requires the output of the neural network to exist, and be non-constant. This motivates the use of activation functions that are smooth, with non-constant derivatives. As seen in  \eqn{jac_defs}, we use the Jacobian of the NSSM for model-based state estimation, motivating not only first-order derivatives to be smooth, but second-order derivatives also. This motivates the use of the \texttt{swish} activation function~\cite{ramachandran2017swish}, 
\begin{equation}\label{eq:swish}
\varsigma(z) = \frac{z}{1 + \exp(-\beta z)},
\end{equation}
which is non-monotonic, smooth, and bounded only from below: see~\fig{swish}; usually $\beta$ is a trainable parameter, but we set $\beta=1$ for simplicity. 
The \texttt{swish} function avoids output saturation since it is unbounded from above, thereby minimizing near-zero gradients for large values similar to the ReLU function.
Also, the \texttt{swish} activation exhibits a soft self-gating property that does not impede the gradient flow  due to its smoothness and non-monotonicity, thereby stabilizing the training procedure~\cite{ramachandran2017swish}.
Next, we show some properties of the \texttt{swish} activation functions that are critical to our task. We begin with the following lemma.
\begin{lemma}\label{lem:1}
Let $\mathcal P_1$ and $\mathcal P_2$ be polynomials in $z$ and $e^z$, respectively. Then 
\begin{equation}\label{eq:lem1}
\frac{d}{dz}\big(\mathcal P_1\otimes \mathcal P_2\big) = \sum_{i\in \mathcal I} c_i z^{q_i} e^{\tilde q_i z} =: \sum_{i\in \mathcal I} c_i \tilde P(z, e^z),
\end{equation}
where $\otimes$ is the Kronecker product, $\mathcal I$ is some index set, and $c_i, q_i, \tilde q_i$ are scalar coefficients.
\end{lemma}

\begin{proof}
Since $\mathcal P_1(z) = \sum_i c_{1,i} z^i$ and $\mathcal P_2(e^z) = \sum_{j} c_{2,j} e^{jz}$, for some set of coefficient vectors $c_1$ and $c_2$, their tensor product can be written as
$\sum_{i} \sum_{j} c_{1,i} c_{2,j} z^i e^{jz}$. 

Applying the chain rule yields the derivative
\begin{align*}
\frac{d}{dz}\big(\mathcal P_1\otimes \mathcal P_2\big) 
&=  \sum_{i} \sum_{j} c_{1,i} c_{2,j} e^{jz} \left(i z^{i-1}+ j z^i \right),
\end{align*}
whose terms can be rearranged to yield~\eqref{eq:lem1}.
\end{proof}

Next, we show that the $n$-th derivative of the \texttt{swish} activation has an alternative representation that enables us to infer its smoothness properties.
\begin{theorem}\label{thm:2}
Let $\varsigma(z)$ be defined as in~\eqref{eq:swish}. For any $n\in \mathbb N$, the $n$-th derivative of $\sigma$ can be written as
\begin{equation}
\varsigma^{(n)}(z) = \frac{\sum_i c_i^{(n)} \left( \mathcal P^{(n)}_1(z)\otimes \mathcal P^{(n)}_2(e^z)\right)}{(1+e^z)^{n+1}},    
\end{equation}
where $\mathcal P^{(n)}_1$ and $\mathcal P^{(n)}_2$ are polynomials in $z$ and $\exp(z)$, respectively and $c_i$ are scalar coefficients.
\end{theorem}

\begin{proof}
We prove this by induction. For $n=1$, we calculate
\[
\varsigma^{(1)}(z) = \frac{(z+1)e^z + e^{2z}}{(1+e^z)^2}.
\]
Therefore, $\mathcal P^{(1)}_1 = [1, z]^\top$ and $\mathcal P^{(1)}_2 = [e^z, e^{2z}]^\top$, with $c^{(1)}=[1,1,1,0]$.

Now, we assume that the statement holds for $n$, and it remains to be shown to hold for $n+1$. To this end, we differentiate $\varsigma^{(n)}$, which yields

\begin{align*}
\varsigma^{(n+1)}(z) &= -(n+1) e^z\frac{\sum_i c_i^{(n)} \left( \mathcal P^{(n)}_1(z)\otimes \mathcal P^{(n)}_2(e^z)\right)}{(1+e^z)^{n+2}}  + \frac{\sum_i c_i^{(n)} \frac{d}{dz}\left( \mathcal P^{(n)}_1(z)\otimes \mathcal P^{(n)}_2(e^z)\right)}{(1+e^z)^{n+1}}, \\
&= \frac{-(n+1)e^z\sum_i c_i^{(n)} \left( \mathcal P^{(n)}_1\otimes \mathcal P^{(n)}_2\right) + (1+e^z)\sum_{\mathcal I} \tilde c_i \tilde{\mathcal P}(z,e^z)}{(1+e^z)^{n+2}},
\end{align*}
where $\tilde c_i \tilde{\mathcal P}(z, e^z)$ is obtained using Lemma~\ref{lem:1}. We observe that the numerator has introduced additional terms that contain higher-order terms of $e^z$ and $z$ along with their combinations. Therefore, by introducing new polynomials $\mathcal P^{(n+1)}_1(z) = [\mathcal P^{(n)}_1(z); z^{n+1}]$ and $\mathcal P^{(n+1)}_2(z) = [\mathcal P^{(n)}_2(e^z); e^{(n+1)z}]$ and choosing appropriate coefficients, one can rewrite the numerator in the previous expression as
\[
\frac{\sum_i c_i^{(n+1)} \left( \mathcal P^{(n+1)}_1(z)\otimes \mathcal P^{(n+1)}_2(e^z)\right)}{(1+e^z)^{n+2}},
\]
which affirms the $n+1$-case, and concludes the proof.
\end{proof}

From Theorem~\ref{thm:2}, we conclude that the $n$-th derivative of $\varsigma(\cdot)$ is composed of differentiable functions, which indicates that its derivatives exist, and are themselves smooth, ensuring the gradient operations to obtain $y$ from $\varphi$ are valid. Consequently, the output path of the NSSM is $n$-times differentiable, as it is composed of smooth functions passed through affine operators (defined by the weights and biases of the NSSM layers). This is useful not only for expressing more complex vector field constraints, but also imperative in the EKF context because one has to evaluate Jacobians of the NSSM as seen in \eqn{jac_defs}, which involves twice-differentiating the output path. 

Additionally, Theorem~\ref{thm:2} implies that, as long as some $c_i^{(n)}$'s are non-zero, the derivatives of the swish function are non-constant. This is also particularly useful since constant higher-order derivatives would not be useful for expressing non-constant vector fields. We hypothesize that as long as the trained weight matrices of $\ec$ and $\dc_y$ are non-zero, and not pathological in terms of dropping rank, one can expect that the output path $\vgrad\dc_y\circ\ec$ of the proposed neural architecture is non-constant if the NSSM uses swish activation functions. 
In fact, we can explicitly check that for $n\le 2$, the swish function and its derivatives are all smooth and non-constant; see~\fig{swish}.
Training the neural SSM involves minimizing the  differentiable training loss as in~\eqref{eq:nssm-loss-fn-2}.

\begin{remark}
Alternatives to the \texttt{swish} activation function include the \texttt{tanh}, \texttt{SELU}, and \texttt{mish} functions, all of which are smooth, and have non-constant first- and second-derivatives. Unlike \texttt{tanh}, \texttt{swish} is unbounded from above, which prevents vanishing gradients and therefore improves gradient flow during backpropagation through deeper network architectures.  Conversely, ReLU or LeakyReLU functions cannot be used as their 2nd-derivatives are zero.
\end{remark}

\section{Benchmarking and Ablations}
\label{sec:example}
\subsection{Example 1. The Bouc-Wen Hysteretic System Benchmark}
\label{sec:bw_example}

Our first example is the Bouc-Wen oscillator which exhibits hysteretic dynamics~\cite{noel2016hysteretic}; a well-known system identification benchmark. Code and data are available at: \texttt{github.com/merlresearch/MetaLIC}. The Bouc-Wen model is described by discretizing the continuous dynamics
\begin{subequations}
\label{eq:boucwen_model}
\begin{align}
\dot{x}_1 &= x_2, \quad
\dot{x}_2 = \frac{1}{m_L}(u - k_L x_1 -  c_L x_2 - x_3), \\
\dot{x}_3 &= \alpha x_2 - \beta\left(\gamma |x_2||x_3|^{\nu - 1}x_3 + \delta x_2 |x_3|^\nu\right), \\
y &= x_1 + w_y,
\end{align}
\end{subequations}
where $u$ is an input excitation acting on the oscillator, $z$ is the hysteretic force, and $y$ is the output displacement that can be measured in real-time. The noise $w_y$ is Gaussian white noise with a bandwidth of 375~Hz, and a standard deviation of $8\times 10^{-3}$~mm. The discretization sampling frequency is 750~Hz. The set of parameters, whose nominal values are provided in~\cite{noel2016hysteretic}, are unknown for our purposes, but assumed to lie within a range $[\Theta_{\min}, \Theta_{\max}]$  given by Table~\ref{tab:bw_params}. 
\begin{table}[!ht]
\caption{Target parameter $\theta^\star$ and ranges for Example~1.}
\label{tab:bw_params}
\centering\small
\begin{tabular}{c||ccccccc}
\toprule
$\theta:=$ & $m_L$ & $c_L$ & $k_L$ & $\alpha$ & $\beta$ & $\gamma$ & $\delta$ \\
\midrule
$\Theta_{\min}$ & 1 & 5 & $2.5\times 10^4$ & $2.5\times 10^4$ & $500$ & 0.5 & -1.5 \\
$\Theta_{\max}$ & 3 & 15 & $7.5\times 10^4$ & $7.5\times 10^4$ & 4500 & 0.9 & -0.5 \\
$\theta^\star:=$ & 2 & 10 &$5.0\times 10^4$ & $5.0\times 10^4$ & 1000 & 0.8 & -1.1\\
\bottomrule
\end{tabular} 
\end{table}

For the target system, parameterized by $\theta^\star$, the benchmark dataset contains a context and target set comprising 40960 samples and 8192 samples, respectively. These samples are noisy, and have been obtained using a multi-sine sweep; see~\cite{schoukens2020initialization} for more details on the data collection procedure. In order to test the performance of meta-learned system identification on this system, we extend this example to a family of oscillators by sampling $\Nsrc=40$ sets of parameters uniformly within the admissible range $\Theta$. These source systems are used for meta-training, and the source dataset is collected by simulating each of these systems for $T_\ell = 10^4$ time steps, with a sampling period of $1/750$s. The excitation to all source systems is $u(t) = 120\sin(2\pi t)$. We demonstrate in Supplementary Fig.~S2 that the source systems exhibit a wide range of hysteretic loops, and therefore, the meta-learning problem is well-posed, as the dynamics are similar, but not identical, across source systems. 

The NSSM architecture has the form~\eqref{eq:nssm_case1} since we have access to $(u, y)$ data. We set the backward window length of $H=20$. The encoder is 7 layers deep, activated by exponential linear unit (ELU) functions, with $[256, 128, 128, 64, 64, 32, 32]$ neurons across layers. The latent dimension $n_\psi=16$, and the state-transition operator $\mathcal A_\psi$ is fully connected 5 layers deep with 256 neurons each and ELU activation. The decoder structure is a mirror reflection of the encoder, and $\mathcal P$ is the identity operator, because no physical constraints are required to be enforced in the benchmark system. For the MAML algorithm, we set inner- and outer-loop learning rates to be $0.001$, the adaptation budget to be $M=40$ iterations, and run the algorithm for 10K epochs with a 80/20 validation split in the source dataset. With NVIDIA 3090X GPU acceleration, the full meta-training phase takes 16~hr. 
\begin{figure}[!ht]
\centering
\includegraphics[clip=true, width=.8\columnwidth]{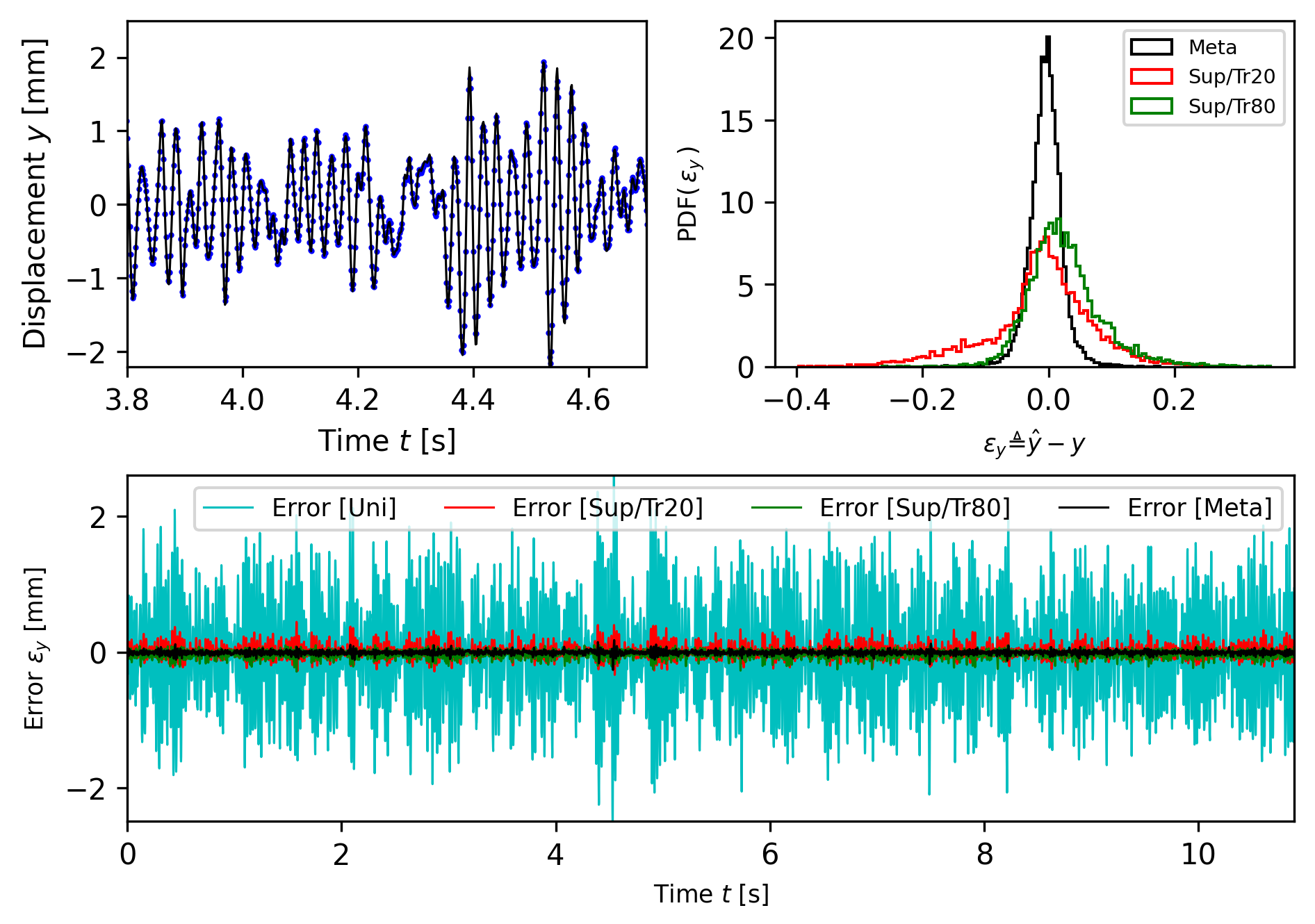}
\caption{Performance of meta-learning for Example 1.}
\label{fig:bw_performance}
\end{figure}

We present the performance of the meta-learned NSSM in Fig.~\ref{fig:bw_performance}. The top left subplot shows that the final MAML-NSSM predictions (blue circles) are overlapping with the ground truth data (black line) for a subset of time where the Bouc-Wen system transitions between two magnitudes of oscillation, between [-1,1] within 3.8--4.2~s, and between [-2,2] within 4.8--5.2~s. Both phases of oscillatory response are reproduced with small predictive error, as evident by the top-right subplot which shows a distribution of the prediction error for three algorithms including MAML. Clearly, meta-learned NSSM errors are tight around 0 with a 99\% confidence interval of $[-0.1, 0.1]$. This is further supported by the lower subplot, which shows the error signal evolution over time (black line) is very close to zero most of the time. 

\begin{table}[!ht]
    \small   
     \caption{\underline{Mean} performance comparison of meta-learned NSSM with supervised learning.}
    \label{tab:bw_results}
    \centering
    \begin{tabular}{l|cccc}
     & \textsc{Meta} & \textsc{Univ} & \textsc{Sup/Tr20} & \textsc{Sup/Tr80} \\
    \hline 
    RMSE [$\times 10^{-5}$ mm] & $\mathbf{2.83}$ & $72.35$ & $8.41$ & $6.50$ \\
    Fit [\%] & $\mathbf{95.47}$ & $29.12$ & $86.76$ & $90.67$    
     \end{tabular}
\end{table}
As reported in Table~\ref{tab:bw_results}, we also compare against: (i) a Universal NSSM, which is trained on all the data, both from the source and target systems, along with greedy supervised learned NSSM with just the target data seeing (ii) 80\% of the target data, called \textsc{Tr80}, or (iii) 20\% of the target data, called \textsc{Tr20}. The objective of comparing against \textsc{Tr80} is to demonstrate that MAML-NSSM can outperform even when almost the entire dataset is available for training, because the target dataset itself is of limited size. The \textsc{Tr20} case is to compare performance when the same amount is given for adaptation/learning, to demonstrate the benefit of using commonalities from similar systems versus learning from scratch. We evaluate these models using the goodness-of-fit and root-mean-squared-error\footnote{We define these formally as
\[
    \textsf{Fit}(\%) = 100 \left(1 - \frac{\|y - \hat{y}\|}{\|y - \bar{y}\|}\right), \;
    \textsf{RMSE} = \sqrt{\frac{1}{N} \sum_{i=1}^{N} (y_i - \hat{y}_i)^2},
\]
where \( y_i \) is the true value, \( \hat{y}_i \) is the predicted value, and \( N \) is the total number of samples.}; mean metrics are presented over 20 independent runs. Clearly, MAML outperforms all 3 ablations, with the closest competitor being \textsc{Tr80}, as expected, because it has access to the most amount of target data. The universal model performs the worst, as it is unable to capture the dynamics of the target system well, due to the large variability in the source systems.

Furthermore, we compare the performance of various meta-learning algorithms for this benchmark problem, including FO-MAML~\cite{finn2017model}, iMAML~\cite{rajeswaran2019meta}, and Reptile~\cite{nichol2018reptile}. As this is a well-explored problem in the nonlinear system identification literature, we also compare against the dynoNet algorithm~\cite{forgione2021a}, the polynomial-basis algorithm proposed in~\cite{esfahani2017polynomial} specifically for hysteretic systems, and the method in~\cite{belz2017automatic} that uses  nonlinear finite impulse response (NFIR) models, nonlinear auto regressive with exogenous
input models (NARX); and nonlinear orthornormal basis function (NOBF). We only present RMSE comparisons in Table~\ref{tab:bw_benchmarks} as those are explicitly reported in all the competitor papers. 

\begin{table}[!ht]
\small
    \caption{Mean RMSE comparison of meta-learning algorithms for the Bouc-Wen system.}
    \label{tab:bw_benchmarks}
    \centering
    \begin{tabular}{l|c|l|c}
    Algorithm & RMSE [mm] &Algorithm & RMSE [mm] \\
    \hline 
    \textsc{MAML}~\cite{finn2017model}  & $2.53\times 10^{-5} $&
    \textsc{FO-MAML}~\cite{finn2017model} & $2.74\times 10^{-5} $\\
    \textsc{iMAML}~\cite{rajeswaran2019meta} & $2.97\times 10^{-5}$ &
    \textsc{Reptile}~\cite{nichol2018reptile} & $2.66\times 10^{-5} $\\
    \textsc{dynoNet}~\cite{forgione2021dynonet} & $4.52\times 10^{-5}$&
    \textsc{Polynomial}~\cite{esfahani2017polynomial}  & $1.87\times 10^{-5}$\\
    \textsc{NFIR}~\cite{belz2017automatic} & $16.4\times 10^{-5}$ & \textsc{NARX}~\cite{belz2017automatic} & $17.3\times 10^{-5}$
    \end{tabular}
\end{table}

From the table, we deduce that the expressivity of neural SSMs in general is higher than NFIR and NARX for this system, with the exception of the method of~\cite{esfahani2017polynomial}, which exhibits excellent performance. This is because the polynomial basis method is tailored specifically to the Bouc-Wen (and generally, hysteretic) systems and, more importantly, their results are reported for the noise-free case $w_y\equiv 0$. Overall, MAML performs well on this problem, even though our modeling paradigm is generalizable to a wide class of systems. The variants of MAML (iMAML, first-order-MAML and Reptile) perform comparably, although we have found empirically that iMAML hyperparameters are  harder to choose than the others, and iMAML performance is quite sensitive to this selection.

\subsection{Example 2: Investigating Subnetwork Fine-Tuning with ANIL on Chaotic Oscillators}
We validate our meta-learned neural SSM on a family of parameter-uncertain unforced van der Pol oscillators, where each oscillator is given by the continuous dynamics
\begin{align}
\label{eq:vdp}
\dot x_1 &= x_2, \quad \dot x_2 = \theta x_2 (1-x_1^2) - x_1, \quad y= x,
\end{align}
which are discretized by a forward-Euler method with a sampling period of $0.01$s. 

\subsubsection{Data collection and implementation details}
The source and target systems are generated by sampling $\theta$ uniformly from the (unknown) range $\Theta=\mathcal U\big([0.5, 2]\big)$, which induces a broad range of dynamics since $\theta$ is effectively a damping parameter. In particular, we simulate until $T=20$ and collect output data for $N_s=200$ source systems with unique $\theta\sim\Theta$ values for each system. The target system $\theta^\star=1.572$, which is also unknown. For each of the source systems, the initial state is randomly sampled from $\mathcal U\big([-1, 1]^2\big)$, while the target system initial state is fixed at $[1, -0.5]^\top$; the final time is sampled randomly from $T\sim \mathcal U([10, 40])$~s, with the sampling period kept constant for all systems. The past window length $H=10$ and predictive window length $H_p=5$. 

Our neural SSM has an encoder consisting of an input layer that takes $x$, and passes it through 5 hidden layers, each of which has 128 neurons, to an output layer that generates a latent variable of dimension $n_z=128$. The entire network uses rectified linear units (ReLUs) for activation. The deep neural network is implemented entirely in \texttt{PyTorch}~\cite{pytorch}.  The batch size $B=32$, with the number of inner-loop iterations fixed at $m=10$, and the total number of meta-training epochs is set to $10^4$. The step-sizes for MAML are $0.01$ and $0.001$ for the inner and outer loops, respectively. Online, the MAML is allowed $m=40$ adaptation steps.
\begin{figure}[!ht]
\centering
\includegraphics[clip, width=.8\columnwidth]{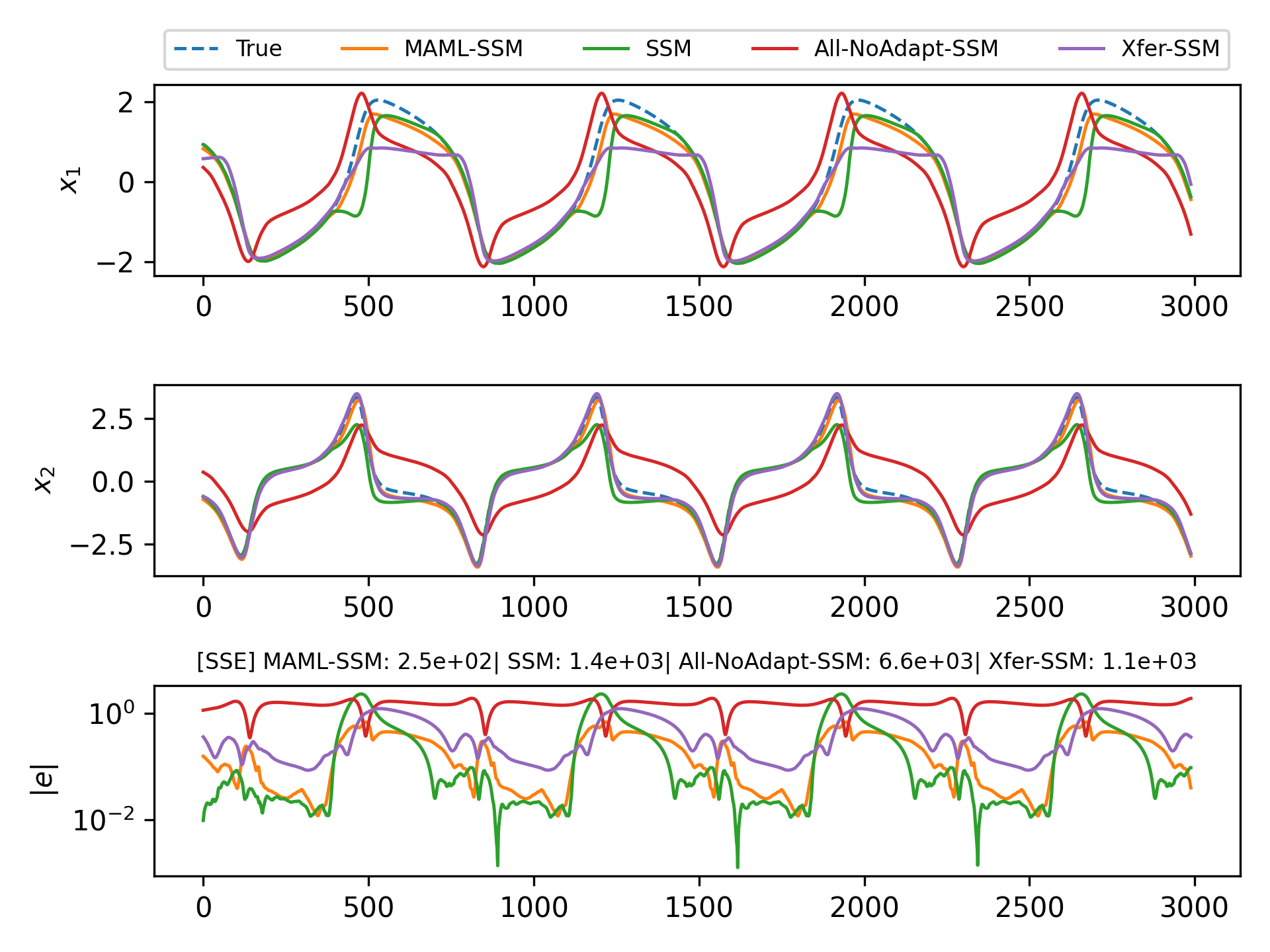}
\caption{Comparison of MAML-SSM with baselines. (\textit{upper, middle}) State $x_1$ and $x_2$ of the oscillator. (\textit{lower}) Comparison of sum-squared-error (SSE).}
\label{fig:baselines}
\end{figure}
\subsubsection{Meta-learning vs. supervised and transfer-learning}
To judge the effectiveness of the proposed \texttt{MAML-SSM}, we compare against a few baselines, all of which have the same architecture as MAML-SSM but are trained differently, using a single training loop rather than an inner-outer bilevel loop as in MAML. These include:
 \texttt{SSM}, which is trained by supervised learning using only the target system data: this is the classical one-training-loop approach for learning neural SSMs. The supervised neural \texttt{SSM} is allowed 10$\times$ more steps to train;  \texttt{All-NoAdapt-SSM}, which is training by supervised learning using all the source system and target system data, without allowing online adaptation steps; \texttt{Xfer-SSM}, which is trained by transfer learning; that is, supervised learning on all the source system data, and allowing a few online adaptation steps exploiting the target system data. The same number of adaptation steps are allowed for \texttt{MAML-SSM} and \texttt{Xfer-SSM}.
To provide some intuition, the first baseline \texttt{SSM} is selected to understand whether \texttt{MAML-SSM} can enable learning from similar systems, or whether using data from other systems negatively affects predictive quality. The second baseline \texttt{All-NoAdapt-SSM} is selected to study whether supervised learning would be enough to learn a good predictive model if it was provided all the data available, both from source and target systems. The third baseline \texttt{Xfer-SSM} is to test whether meta-learning can outperform transfer-learning in few-shot, data-poor situations.

The performance of \texttt{MAML-SSM} and the baselines is illustrated in Fig.~\ref{fig:baselines}. At online inference, the context set is generated by using data from the first 400 time steps, and the target set to be predicted by the neural SSM is the next 3000 time steps. The top and middle subplots show the evolution of the states of the system, with the dashed line denoting the true query system states. The lowest subplot shows the evolution of the sum-squared-error (SSE) between the true state and the predicted state over the 3000 time steps. It is immediately clear that \texttt{MAML-SSM} incurs the smallest prediction SSE and outperforms the baseline, with the transfer learning \texttt{Xfer-SSM} performing second best with double the SSE, closely followed by \texttt{SSM}, and with greedily using all the data as in \texttt{All-NoAdapt-SSM} exhibiting significantly worse performance uniformly across time. From the $x_1$ subplot, we can reason that \texttt{SSM} most likely overfits the context data, and therefore, it exhibits excellent predictive accuracy in within-training-set points of time, with marked deterioration at out-of-set points. Conversely, \texttt{All-NoAdapt-SSM} severely underfits the data, which is plausible, as it inherently tries to find a set of neural SSM weights to fit the `average' dynamics induced by the source and query systems, rather than adapting to the query system. As expected, meta- and transfer-learning exhibit good performance, but our proposed \texttt{MAML-SSM} outperforms \texttt{Xfer-SSM}. The paucity of query system data and the few number of adaptation iterations require the trained learner (before adaptation) to have a set of neural weights that can rapidly adapt to the query system: since this is explicitly what MAML is trained for, the \texttt{MAML-SSM} has the ability to rapidly learn a good predictive model for the given example. In contrast, the transfer learning approach (before adaptation) is not explicitly trained keeping the adaptation in mind, therefore the set of initial weights is likely tailored to generating good predictions for the family of source systems, and the few-shot nature of the adaptation is not enough to enable \texttt{Xfer-SSM} to learn a model as accurate as \texttt{MAML-SSM}; this is especially notable from the $x_1$ plot, where transfer-learning completely fails to capture the positive half of the oscillatory dynamics.

\subsubsection{Subnetwork Fine-Tuning with ANIL}
ANIL can be used to partially meta-learn neural SSMs by splitting the neural SSM into a base-learner (layers with parameters $\omega_l \notin \omega_{\sf in}$) and a meta-learner (layers with parameters $\omega_l \in \omega_{\sf in}$), wherein only the meta-learner is adapted online; the base-learner is fixed. While we understand that the contribution of the encoder $\mathcal{E}$ is to lift or project the inputs to a latent space which provides benefits in terms of computational tractability or expressivity, and the role of the state-space matrix is to encode the dynamics in the latent space, it is not immediately clear which should be the base-learner and which layers should be adapted. 

\begin{figure}[!ht]
\centering
\includegraphics[width=.9\columnwidth, clip]{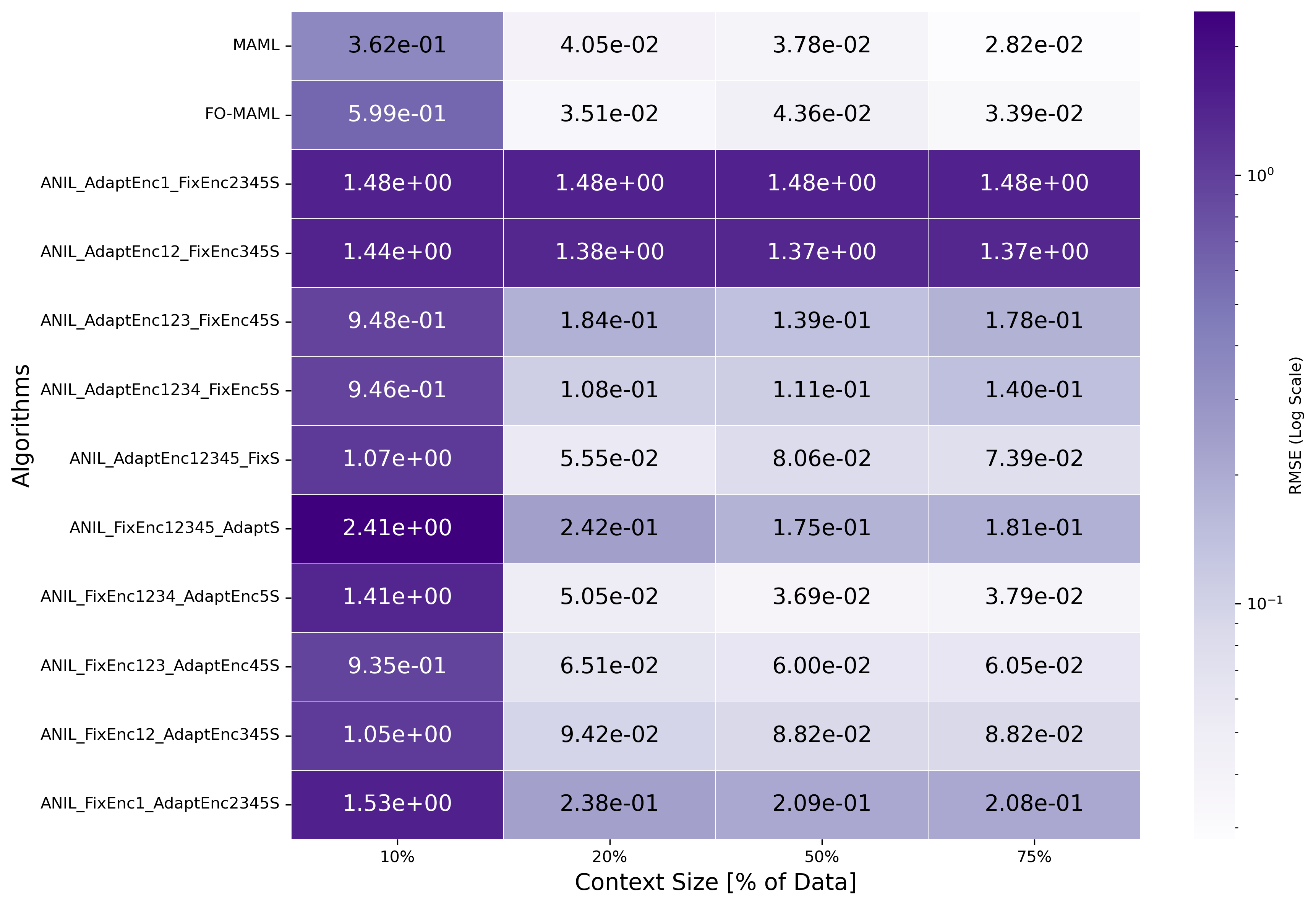}
\caption{RMSE (averaged over 20 runs) heatmap produced by tuning subnetworks with ANIL.}
\label{fig:anil}
\end{figure}
To answer this question, we use the target system with a fixed final time $10$s, and allow various subnetworks of the model to behave as the base-learner, and other layers to adapt. The decoders are always allowed to adapt in order to maintain consistency with the adapted encoder, i.e. so they can retain reconstructive ability. The experiment results are shown in Fig.~\ref{fig:anil}, where the columns of the table indicate what percentage of the testing data was used for adaptation, the colorbar represents the color-code for the RMSE (lighter colors indicate better accuracy of the NSSM), and the rows indicate which meta-learning algorithms were considered. Of course, MAML and first-order (FO)-MAML allow for full fine-tuning, i.e. all layers are adapted. For the last 10 rows, we report the results of ANIL using the following nomenclature: \texttt{ANIL\_FixEnc123\_AdaptEnc45S} to indicate that layers \texttt{1,2,3} of the encoder are fixed (\texttt{1} being nearest to the input channels) and only the \texttt{4,5} layers of the encoder along with the state-transition layers $\texttt{S}=\mathcal{A}_\psi$ are adapted. Similarly, \texttt{ANIL\_AdaptEnc12345\_FixEncS} indicates that all encoder layers were adapted and the state-transition layers were fixed. Clearly, from Fig.~\ref{fig:anil}, we can see that the best performance is for MAML and FO-MAML, which allows full fine-tuning. This is to be expected, as this allows for the most expressive adaptation. Interestingly, we see that the worst performance is obtained when the encoder is mostly fixed and only the state-transition layers are allowed to be adapted. This indicates that it is most critical to allow the encoder to learn and adapt the lifting/projection function that determines the latent space rather than just the state-transition matrix. Thinking about this in another way, it is not enough for just $\mathcal A_\psi$ to be adapted while fixing most of $\mathcal E$ if we are to learn across similar systems with some variety in the dynamics. 

It is also important to note that even though meta-learning is effective for few-shot learning, there is a practical lower bound in terms of the amount of data that is enough for a given problem: here, 10\% target data is too small for meta-inference and consistently gives poor results, whereas anything more than 10\% yields decent performance, although the best performance overall is with 50\% or more data for most high-performing algorithms. Of course, since we report the mean in the colortable, it is possible that there is slightly higher error when more data is provided, but the discrepancy is always small and most likely due to random seeds. However, one reason why adding more data might be resulting in worsened performance without full fine-tuning is due to overfitting, e.g. in the cases of \texttt{ANIL\_AdaptEnc123\_Fix45S} and  \texttt{ANIL\_AdaptEnc1234\_Fix5S}. The best performance is reported  when the first few layers of the encoder is fixed and the final layers are adapted along with adaptation of $\mathcal A_\psi$. This can be defended by our knowledge of deep neural network training, wherein the early layers are responsible for extracting basic, general features from the input data, such as common features and and simpler transformations, whereas later layers focus on more complex, system-specific transformations~\cite{yosinski2014transferable}.

\section{Case Studies with Constrained Physics}
\label{sec:case_studies}

We study two real-world systems, detailing the challenges encountered in these applications and how they benefit from the proposed approach, along with implementation details.

\subsection{Vapor Compression with Unknown Refrigerant Mass}
\label{sub:num_VCS}
Vapor compression systems (VCSs) are crucial in various heat transfer applications due to their versatility and wide range of operating temperature. As illustrated in \fig{vcs_overview_diag}[A], a VCS consists of a variable-speed compressor, a variable-position electronic expansion valve (EEV), two heat exchangers (condenser and evaporator), and two fans that force the air to flow through the heat exchangers. This system transfers heat from the air passing through the evaporator to the air passing condenser via the refrigerant flowing through the system. 
\begin{figure}[!ht]
    \centering
    \includegraphics[width=\columnwidth,clip=true]{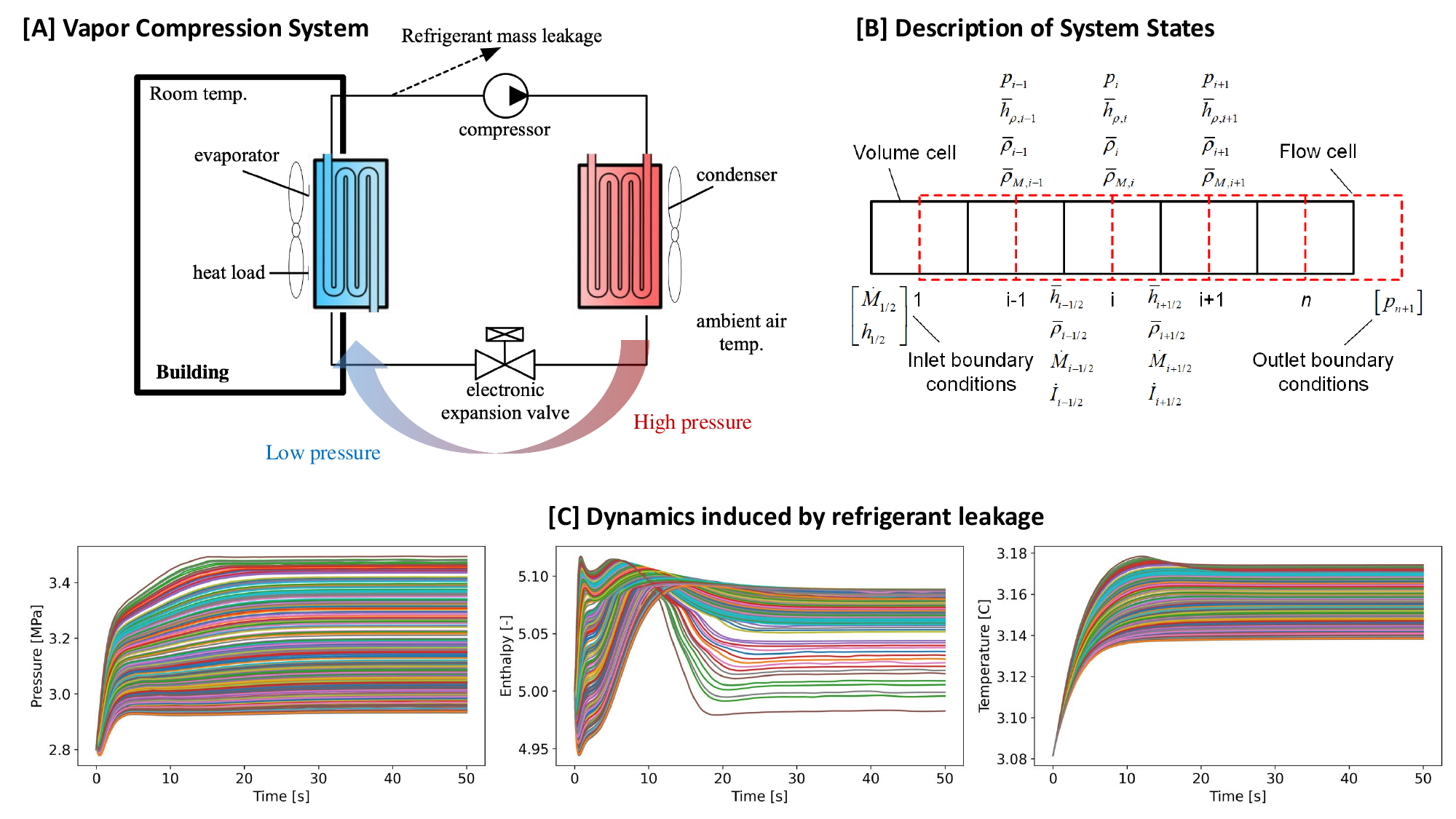}
    \caption{Range of HVAC dynamics induced by variable refrigerant mass.}
    \label{fig:vcs_overview_diag}
\end{figure}

Many technological solutions for VCSs use simulation models and digital twins~\cite{vering2021digitalTwin,laugman2023dt} for predicting their behavior. Accurate predictions of VCSs dynamics are needed for control~\cite{bortoff2019a}, estimation~\cite{chakrabarty2021scalable,deshpande2023MCEKSifac, deshpande2022Constrained}, and performance optimization~\cite{khosravi2019controller,chakrabarty2021accelerating}.
High-fidelity models of VCSs are generally derived based on the physics of the individual components in VCS and the overall system~\cite{qiao2015}. Developing physics-based models of these systems require equation-oriented programming environment and face several computational challenges, even when simplified and combined with efficient data-driven learning modules~\cite{chinchilla2023residual}. While physics-oriented models are necessary for accurate system predictions, their computational complexity becomes a bottleneck for efficient deployment, where rapid forward simulation and frequent model linearization is often required. Therefore, \textit{computationally efficient surrogates} of high-fidelity physics-based models of VCSs which mirror the behavior of underlying complex physical system are critical for analysis and design.

Vapor compression systems of similar technical specifications exhibit similar dynamic behaviors since the governing physics are identical.  
However, each system may exhibit a unique dynamic response due to factors such as unit-to-unit manufacturing variability, differences in installation configurations,
and differences in refrigerant quantity present in the system. 
For the purpose of this study, we particularly focus on the dynamics variability induced by the varying refrigerant quantity across different systems which may arise due to installation errors, refrigerant leakage over time, or refrigerant replenishment during maintenance events. 

Generally, dataset for a particular VCS is typically limited and expensive to obtain as it may require sophisticated instrumentation. However, an entity of interest (e.g., manufacturer or supplier) may have access to a collectively large operational dataset of different VCSs located across many customer sites. Additionally or alternatively, a training dataset can also be generated from high fidelity physics-based models of VCSs. 
Thus, surrogate modeling efforts of vapor compression systems may significantly benefit from the proposed meta-learning approach. 
Large dataset collected across different systems can be used for learning a general NSSM for VCSs, while the limited dataset from a target system can be used for fine-tuning the NSSM weights. 

\subsubsection{Data Generation}
 We use a physics-based simulator of the VCS~\cite{deshpande2022Constrained} for generating the training dataset. 
The model implements two heat exchangers which dominate the overall system dynamics by discretizing them in four finite volumes each, while algebraic models are used for the compressor and expansion valve. Furthermore, a heat exchanger is comprised of interconnecting components that describe the thermal behavior of the pipe wall, the flow of air across the heat exchanger, and the one-dimensional refrigerant pipe-flow, forming an index-1 system of differential algebraic equations (DAEs). The one-dimensional flow of refrigerant through the pipe is modeled using fundamental physical laws, namely, the conservation of mass, momentum, and energy. After order reduction via Pantelides' algorithm, the DAEs are transformed to a set of ordinary differential equations (ODEs). The state vector contains the pressure $\Pres[i]$, specific enthalpy $\Enth[i]$, and temperature $\Temp[i]$ for each finite volume $i\in\{1\cdots 8\}$ and the control inputs to the model are compressor frequency and expansion valve position, thus, $n_x=24$ and $n_u=2$. The measurements include temperatures and pressures at select volume locations in the heat exchangers. In particular, the measurement vector is given by
\begin{align}
y = \begin{bmatrix}
        \Pres[1] & \Temp[2] & \Temp[4] & \Temp[6] & \Pres[8] & \Temp[8]
        \end{bmatrix} \in \Real^{6},
\end{align}
that is, only 6 of the 24 states are directly measured.

The training dataset consisting 250 trajectories was generated using the simulation model such that each trajectory corresponded to a system with a different refrigerant mass. 
This model implements an initialization routine in which all variables are set to their default values. A set of 250 initial conditions to be used for generating the training dataset was constructed by driving the model (with default initialization) to a steady state over a simulation of 50~s with the constant control inputs, as shown in \fig{vcs_overview_diag}[C].  The compressor frequency and EEV position were set to 50~Hz and 300~counts for all runs. This eliminates effects of default initialization to some extent and allows construction of a meaningful, controlled set of initial conditions to test our proposed methods. 
The scenario of varying refrigerant mass was emulated by selecting different pipe lengths for HEXs in different trajectories. We note that, identical default initializations of state variables and different pipe lengths amount to variable refrigerant mass across different trajectories. The set of 250 values for pipe length were randomly sampled from the uniform distribution $\mathcal{U}(20, 40)$~m. The variability of the dynamics due to variation in pipe lengths and refrigerant mass is shown in \fig{vcs_overview_diag}[C] for a representative subset of state variables. 

The model was simulated to generate a set of 250 trajectories over 500~seconds with the afore-mentioned set of initial conditions. The control inputs, compressor frequency (Hz), and expansion valve position were randomly sampled from the uniform distributions $\mathcal{U}(40, 80)$ and $\mathcal{U}(282, 312)$, respectively, to generate a rich dataset with persistently excited inputs. These trajectories were sampled at 0.1~second intervals with sample-and-hold control inputs. 
The set of 250 trajectories was split into 80-20 for source and target dataset, wherein only first the 20 seconds of the target trajectories were used for fine-tuning. Furthermore, 100 source trajectories were split into 60-20-20 for train, test and validation datasets.

High-fidelity simulators used in VCS applications are typically developed based on physics of the system \cite{deshpande2022Constrained}, therefore, they inherently satisfy any fundamental constraints imposed by the physics. An example of such a constraint relevant to VCSs is the pressure gradient constraint, i.e., the non-increasing pressure in the direction of refrigerant flow. In order to ensure consistency with the physical laws, it is imperative that state estimates obtained using a learned NSSM and EKF must also satisfy these physics-informed constraints. We explicitly incorporate these constraints into the NSSM and the filtering algorithm. Specifically, can write the pressure gradient constraint as a linear inequality $G x\!\le\! 0$, and utilize the discussion in Section~\ref{sub:physics_VCS} to improve the estimates.

\subsubsection{Results and Discussion}
State estimation results of the EKF \eqn{deep_cekf} that uses the NSSM \eqn{vcc_model_for_ekf} for predictions are shown in \fig{xy_meta_v_nometa}.
We considered two variations of NSSMs for a comparison. An NSSM whose weights were fine-tuned using Algorithm~\ref{alg:maml_inference} is referred to as Meta-NSSM, whereas NSSM learned without the fine-tuning is referred to as No-Meta-NSSM. 
EKFs using these two variations of NSSMs were tested on target systems. Reference trajectories of one of the target systems and those estimated by EKFs along with the associated confidence bounds are shown in \fig{xy_meta_v_nometa}.

\begin{figure}[htb]
    \centering
    \includegraphics[width=0.8\linewidth]{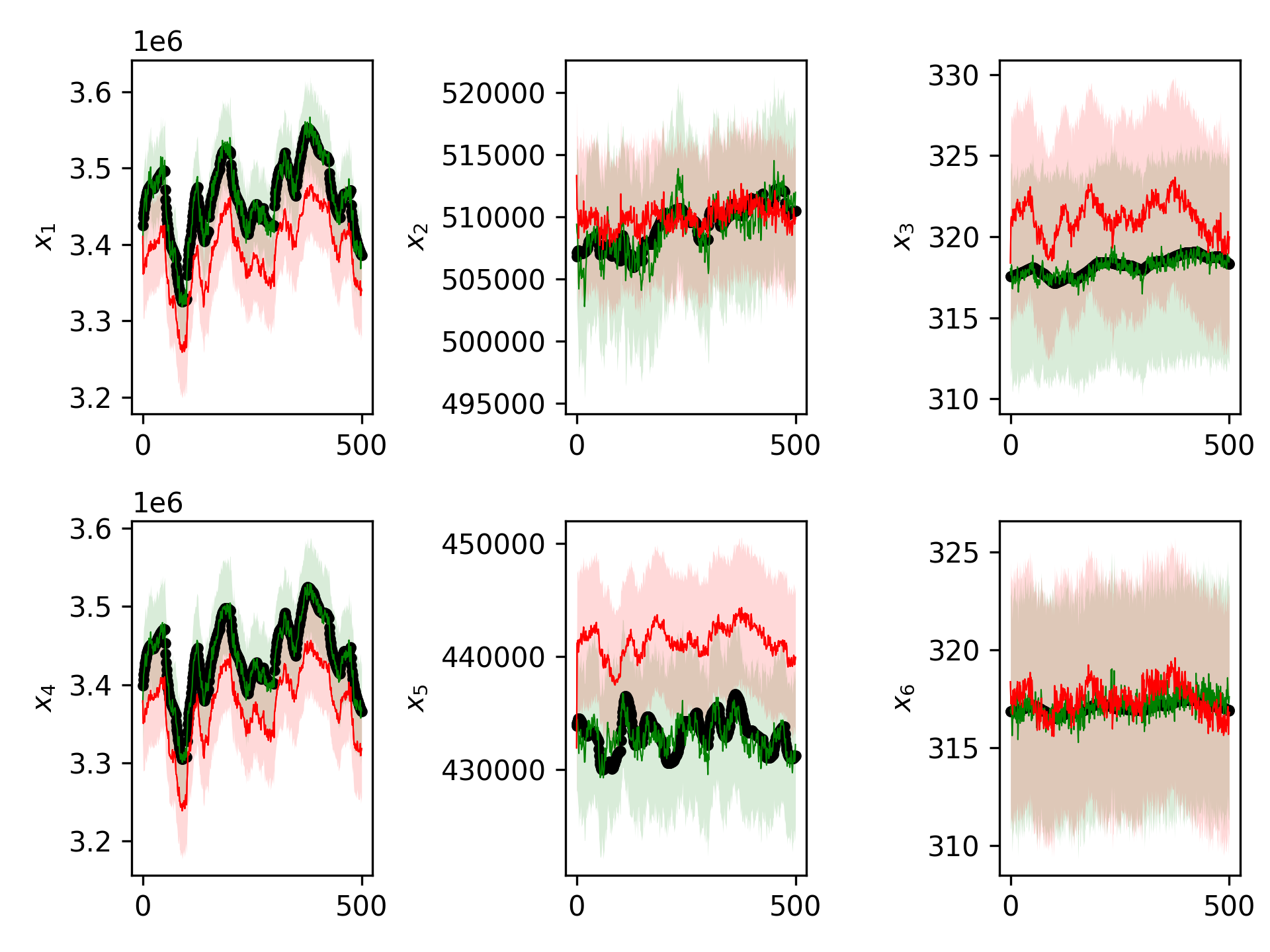}
    \caption{Reference state trajectories (black) and those estimated by EKF using Meta-NSSM (\textcolor{ForestGreen}{green}) and No-Meta-NSSM (\textcolor{red}{red}) as the prediction models. Shaded area around the solid lines denote $3\sigma$-bounds.}
    \label{fig:xy_meta_v_nometa}
\end{figure}

Meta-NSSM-EKF demonstrated significantly smaller estimation errors than No-Meta-NSSM-EKF as observed in \fig{xy_meta_v_nometa}. The superior estimation accuracy of Meta-NSSM-EKF can be primarily attributed to better prediction capabilities of the Meta-NSSM since its weights are fine-tuned or adapted for a particular target system. On the other hand, No-Meta-NSSM-EKF uses a general prediction model learned for a family of VCSs without system-specific adaptations, thus, leading to an overall degraded estimation performance. 

Incorporating physics-informed constraints \eqn{constr_structure_vccs} into the 
NSSM architecture were found to enhance the estimation performance of both NSSMs and EKFs. Although these specific results are not included in this paper for brevity, our findings align with those reported in our previous work \cite{deshpande2023physics}, which discussed the effectiveness of physics-constrained architectures in detail. 

\subsection{{Indoor Localization with Unknown Magnetic Fields}}
\label{subsec:localization}
An another motivating application for combining physics-informed NSSMs, meta-learning, and state estimation is in magnetic-field-based indoor positioning~\cite{Berntorp2023a}.
Magnetic materials present in buildings cause anomalies in the ambient magnetic field \cite{Li2012,Angermann2012}, which can be leveraged for localization.  However, several challenges hinder  high-accuracy positioning using magnetic fields. Most critically, the function mapping the magnetometer position to the magnetic field is unique for each indoor environment, does not have specific structure, and needs to enforce physics-informed constraints (i.e., it has to be a curl-free vector field due to the absence of exciting local currents). Consequently, there is a need to identify the unknown magnetic field, generally using noisy measurements obtained online. With the help of expensive offline instrumentation, we can acquire very limited state-output data  from the actual system, denoted by
$$
\Dtest = \{(x_{0:T^\star}, u_{0:T^\star}, y_{0:T^\star})|\theta^\star\}
$$ 
obtained over the time span $\{0,\ldots, T^\star\}$.
Due to this scarcity of target system data, it is unrealistic to expect that an NSSM trained solely on this data will exhibit satisfactory predictive performance or be able to generalize. Our goal is to identify a NSSM for the target system with limited data, capable of enforcing physics-informed constraints within the NSSM architecture, and to subsequently employ this NSSM for online state estimation despite sensor noise and process uncertainty. Herein, we  demonstrate physics-constrained meta-learning for learning the magnetic field with subsequent state estimation for positioning.

\begin{figure*}[!htbp]
    \centering

        \begin{subfigure}{0.45\textwidth}
        \centering
        \includegraphics[width=\textwidth,clip=true]{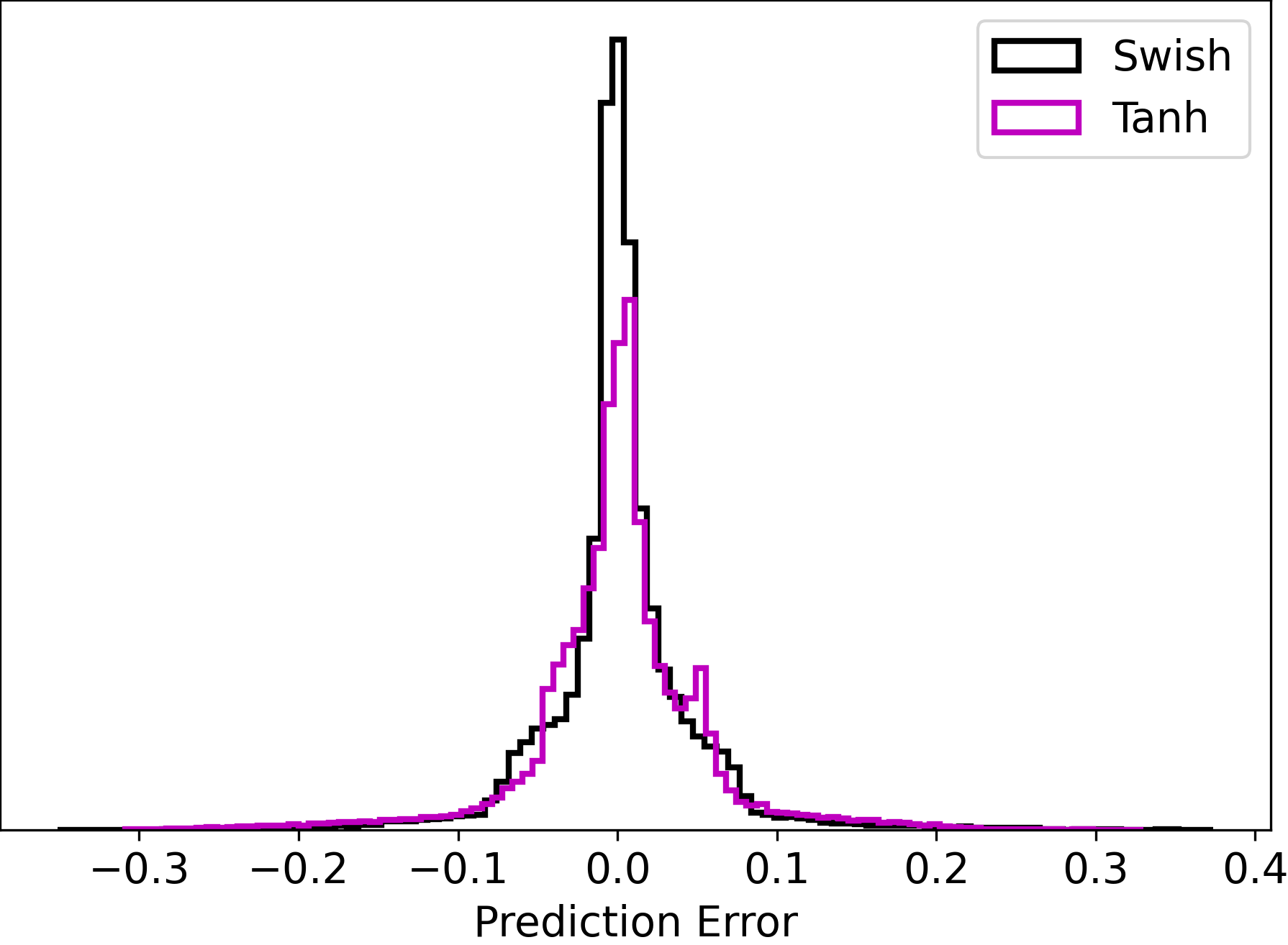}
        \label{fig:image3}
    \end{subfigure}
    \begin{subfigure}{0.45\textwidth}
        \centering
        \includegraphics[width=\textwidth,clip=true]{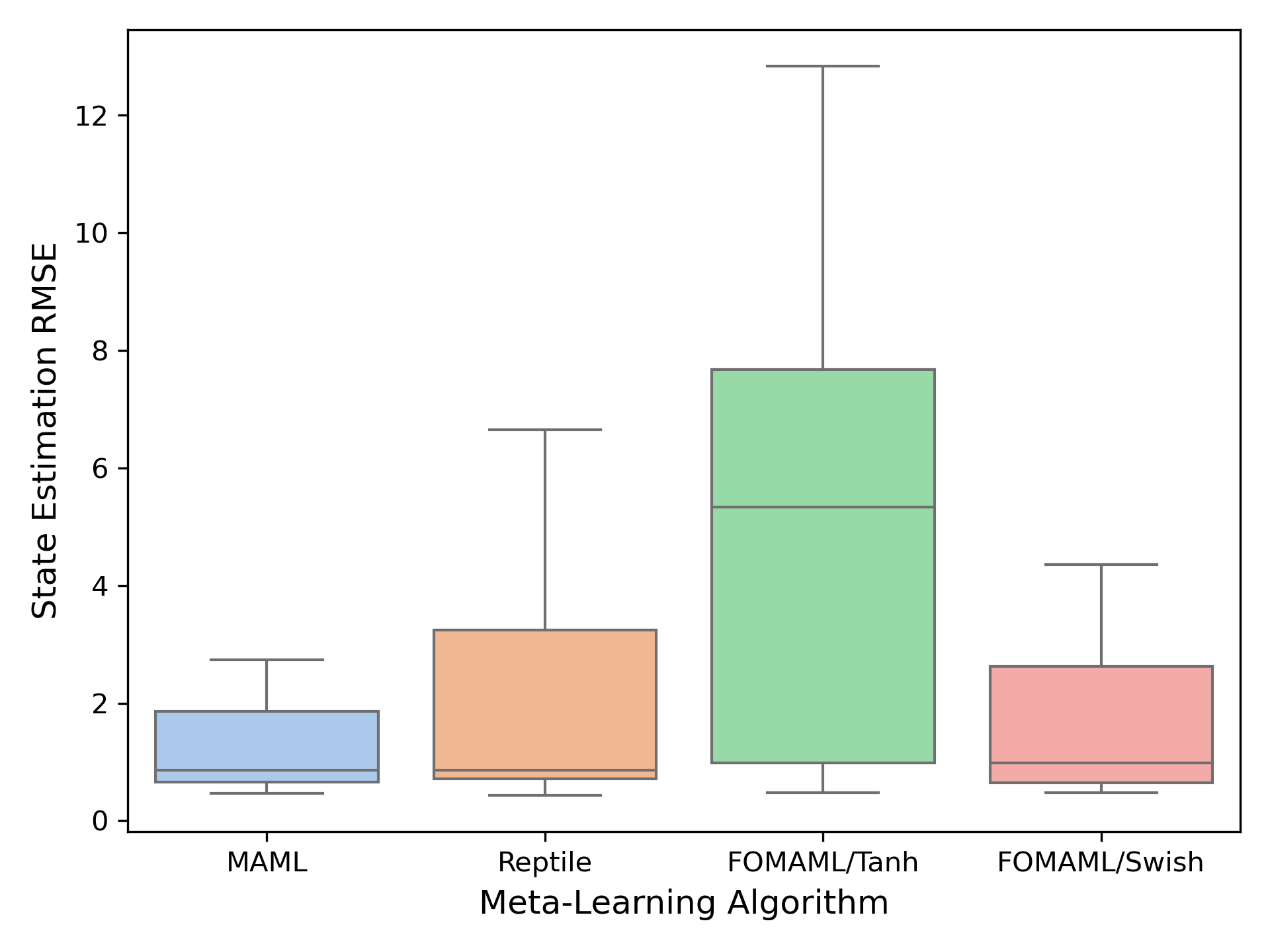}
        \label{fig:image4}
    \end{subfigure}
    
    \caption{(\textit{Left}) Histograms of predictive error distribution between swish and tanh MAML. (\textit{Right}) Boxplots demonstrating state estimation RMSE across different meta-learning EKF for 20 unknown target systems.}
    \label{fig:magfield}
\end{figure*}

\begin{figure}[!ht]
    \centering
    \includegraphics[width=.6\columnwidth]{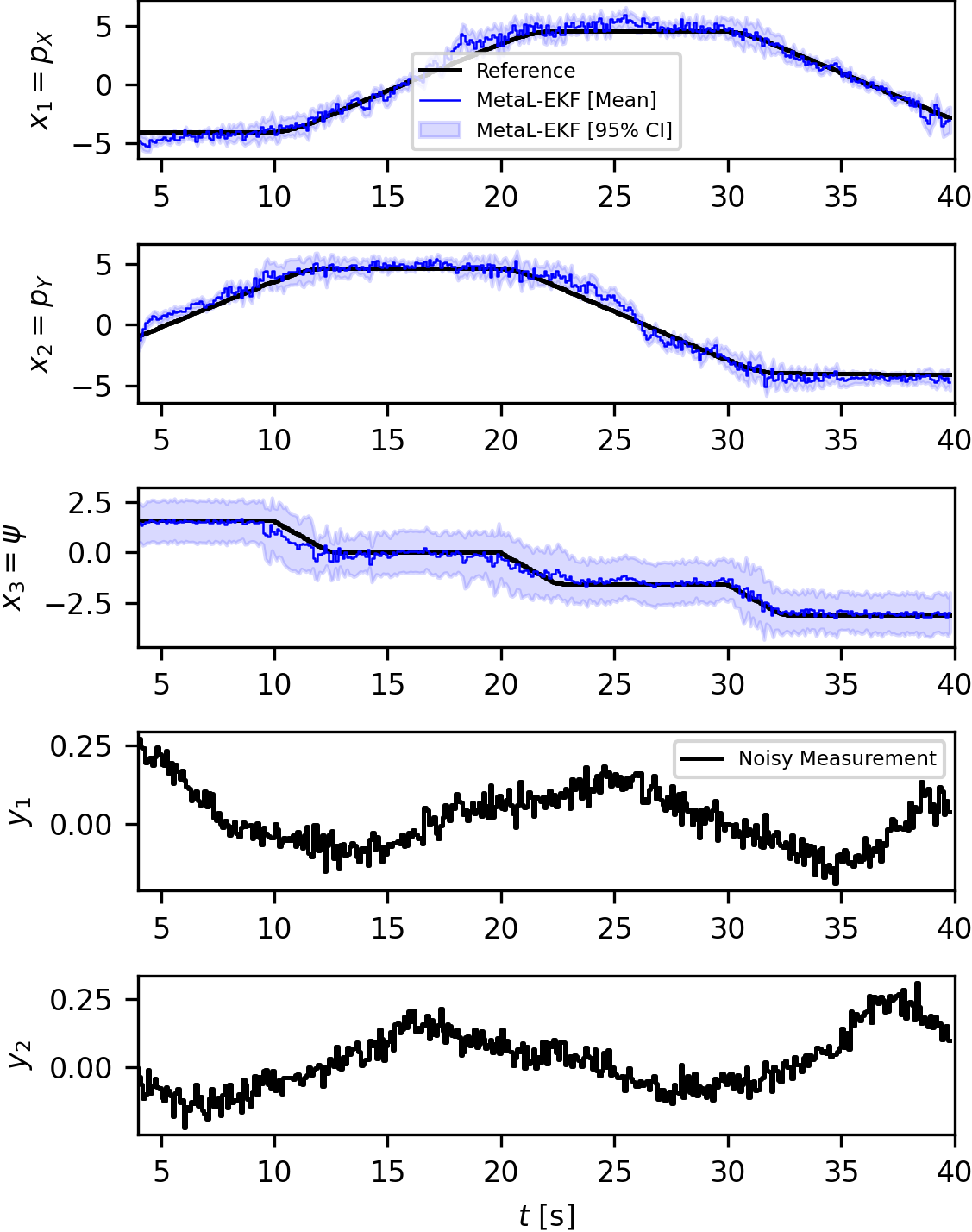}
    \caption{Performance of MetaL-EKF with MAML-FT/Swish for a randomly chosen test vehicle under uncertain magnetic field. Horizontal axis starts at $t=4$~s to reflect that previous data was used as context for meta-inference.}
    \label{fig:states_outputs}
\end{figure}

\subsubsection{System Description}
The scenario we consider is the example used in~\cite{Wahlstroem2013}, tailored to the recursive setting similar to~\cite{Berntorp2023a}. Specifically, the magnetic-field model is from \cite{Wahlstroem2013}, for which we have ground truth and can accurately evaluate our proposed approach. 

Consider a sphere centered at the origin with a radius of $r_0$m having a uniform magnetization of $\mathcal{M} = [m_X, m_Y, 0]^\top$A/m; see Fig.~\ref{fig:magfield}(A). Then, the magnetic-field function $h_{\theta}$ mapping the position $p:=[p_X,p_Y]$ to the magnetic field $y$ is described by
\begin{equation}\label{eq:h_theta}
	h_{\theta} = \left\{\begin{matrix}
		-\mathcal{M}/3 & \text{if } r < r_0 \\
		\frac{m_0}{4\pi}\left(\mathcal{M} /r^3 + 3/r^5 (\mathcal{M}^\top p)p \right ) & \text{if }  r \geq r_0
	\end{matrix}\right .
\end{equation}
where $m_0 = 4/3\pi r_0^3$ and $r = \lVert p \rVert$. With zero-mean Gaussian noise $\eta_t \sim \Ncal (0,\Rmeas_t)$ with covariance $\Rmeas_t$ on the measurements, we  write the measurement equation as
$y_t = h_{\theta}(x_t) + \eta_t$.
We consider a mobile wheeled robot with car-like kinematics described by a kinematic single-track model moving in the $XY$-plane. The kinematic single-track model has three states: the global (planar position)  and the heading angle, $
x = [p_X, p_Y,\psi]$.
The wheel-speed measurements directly provide the velocity $v_X$. The continuous-time model is
\begin{equation}\label{eq:KinematicEquations}
	\dot{x} = \begin{bmatrix}
		v_{X}\cos{(\psi + \beta)}/\cos(\beta)\\
		v^{X}\sin{(\psi + \beta)}/\cos(\beta)\\
		v^{X}\tan{(\delta)}/L
	\end{bmatrix},
\end{equation}
where $L=l_f +l_r$ are the distances from the origin to the front and rear wheel axle, respectively, $\beta = \arctan(l_r\tan(\delta)/L)$ is the kinematic body-slip angle where $\delta$ is the steering angle, and the velocity is related to the wheel speeds by $v_{X}=(\omega_f+\omega_r)R_w/2$; $R_w$ is the wheel radius. After time discretization, we write  \eqref{eq:KinematicEquations} concisely as 
\begin{equation}\label{eq:KinematicST}
	x_{k+1} = f_{\theta}(x_{k},u_t) + w_t,
\end{equation}
with Gaussian zero-mean process noise, $w_t \sim \Ncal(0,\Qproc_t)$, to account for  model mismatch, $u_t= [\delta,\omega_f,\omega_r]$ being control action. The steering controller is a simple pure-pursuit controller tracking a reference path $p_{r}(t)$. In the context of \eqref{eq:nonlinear_system}, the unknown parameters $\theta$ consists of the wheelbase parameters $l_f$, $l_r$, the maximum steering angle $\delta$: these affect the vehicle dynamics; and the magnetization parameters $m_X$, $m_Y$, the magnetic-field radius $r_0$: these affect the magnetic vector field.

\subsubsection{Implementation Details}
In order to construct a meta-learning dataset, we need to generate source and target systems. To this end, we assume that the range of $l_f, l_r \in [0.02, 0.25]$, $m_X, m_Y \in [-1.5, 1.5]$, the maximum steering angle $\delta\in[10, 20]$, and $r_0\in [1.5, 4.5]$. Then we construct 80 source and 20 target systems by randomly sampling parameters from these ranges, and executing simulations that yield $x$, $u$, and $y$ data, for each source and target system. The simulations last for 40s, with a sampling period of 0.1s and identical references to be tracked by the vehicle. Despite the common reference to be tracked, we have observed (we do not show a plot due to limited space) significant variation within $\Dtrain$ and $\Dtest$ due to variation in $\theta$.

After a preliminary neural architecture search, the dimension of the latent space was selected to be $n_{\psi} = 128$. Each of the components of the NSSM was implemented with fully connected layers activated by \texttt{swish} functions. The graph of the encoder and two decoders were selected to be $\ec_{xu}: 6-256-128-128-64-32-32-128$, $\dc_{x}: 128-32-64-128-256-3$ with 3 updated states as outputs, and $\dc_{y}:128-64-64-128-128-256-256-1$, since the output of this is a scalar potential. The state transition operator $\mathcal{A}_{\psi}$ was implemented network with six fully-connected layers of identical hidden dimension of 256.
For meta-inference, the last 2 layers of $\ec_{xu}$ and all the layers of $\mathcal A_\psi$, $\dc_x$ and $\dc_y$ were adapted with target system data. The first few layers of $\ec_{xu}$ were fixed after meta-training. 
We utilize the results of Section \ref{sub:physics_mag} to enforce curl-free constraint on the learned magnetic field. 
We reiterate that in this case study, $\mathcal P_y\equiv\vgrad$ is not a trainable neural network, it merely denotes the vector differential operator that differentiates the output of $\dc_{y}$ with respect to $p$. 

While one may argue that this network is unnecessarily deep and the base-learner could have been constructed with fewer layers, the reason we chose many layers (other than just to improve expressivity across the family of source and target systems) was to investigate whether the \texttt{swish} activation truly avoids vanishing gradients in deep networks, and whether some empirical advantages may be obtained compared to other activation functions such as the \texttt{tanh} function that have been previously considered. 

\subsubsection{Results and Discussion}
The meta-learned NSSM was trained offline on the source dataset for 30000 epochs using the Adam optimizer with fixed learning rates $\beta_{\sf out}=10^{-4}$, $\beta_{\sf in}=10^{-3}$, and $M=10$ adaptation iterations in accordance with Algorithm~\ref{alg:maml}; an NVIDIA 3090X GPU was used for training, and the total training time was at most 22 hours. The weights of the NSSM were saved when the validation loss decreased.
For meta-inference, we use 20\% of the  target system data, which is the initial $T_c=8$s out of the total simulation time of $T_\ell=40$s. The performance of the NSSM was evaluated in terms of prediction and estimation errors for the remaining  $80\%$ of the target system's trajectory never seen by the NSSM. The CPU time for meta-inference adaptation of the non-fixed layers was 0.58 seconds on average.

Our numerical results are presented in~\fig{magfield}, ~\fig{states_outputs}, and Fig.~\fig{magfield_err}. In Fig.~\fig{magfield_err}, we compare the magnetic field estimation error $y - \hat y$ between the  the true field $h_\theta$, obtained by evaluating $h_\theta$ in~\eqref{eq:h_theta} parameterized with
$m_X = 1$, $m_Y=1$, $r_0 = 3$m and the estimated magnetic field using the meta-learned \texttt{swish}-activated (black) and \texttt{tanh}-activated (magenta) NSSM, evaluated over a regular 200 $\times$ 200  grid in a 2-D spatial domain. Clearly, for both estimates, the error is small:  that is,  the learned NSSM closely approximates the magnetic field throughout the domain, with noticeable differences only around a small neighborhood of the region of discontinuity at $r_0=3$m. The \texttt{tanh}-activated NSSM induces larger approximation errors than the \texttt{swish}-activated NSSM, especially between: this is also cross-checked by computing the root-mean-squared error (RMSE) of approximation on the 100 $\times$ 100 grid: for the \texttt{swish} activated NSSM, the RMSE is $6.72\times 10^{-4}$ compared to  $8.03\times 10^{-4}$ for \texttt{tanh}-activation. The distribution of the errors is also shown via histograms in the left subplot of~\fig{magfield}, where we see that the \texttt{swish}-based NSSM has a tighter variance around zero, with many more time-points with estimation errors near-zero. 

We used the meta-learned NSSM as a predictive model for implementing an EKF of the form \eqn{deep_cekf_all}, referred to as the {\it MetaL-EKF}. 
In the right subplot of~\fig{magfield}, we compare the state estimation accuracy of MetaL-EKF implemented with different meta-learned NSSMs. Boxplots of the cumulative state estimation error norm $\sum \|x - \hat x\|$ are reported for 20 unknown target systems. In particular, we consider four variants of meta-learning, two of which are the most popular meta-learning algorithms. The first boxplot (blue) is MAML approach with the proposed \texttt{swish} activation, the second is Reptile (orange) with \texttt{swish} activation, the third (green) is FO-MAML with \texttt{tanh} activation, and the final (pink) is the proposed FO-MAML approach with \texttt{swish} activation. It is immediately clear that the \texttt{swish} activation consistently improves  performance. The full fine-tuned MAML outperforms the rest, although the Reptile and FO-MAML algorithms are comparable in average performance, with FO-MAML showing better robustness (tighter confidence intervals) across target systems. 

The states and outputs for a particular target system (chosen randomly) is shown in Fig.~\ref{fig:states_outputs}, along with their estimates and 95\% confidence intervals, with 10\% data used for adaptation (first 4 seconds). MetaL-EKF demonstrates good estimation accuracy for the most part of the unseen trajectory. Relatively larger deviations of estimated state trajectories from the true trajectory and larger state error covariances at some points can be attributed to sharp turns in the vehicle trajectory. Despite these infrequent deviations, MetaL-EKF quickly decreases the gap between the estimated and true trajectories as the vehicle moves past the turning maneuvers, which can be attributed to the predictive capability of the meta-learned MAML+FT/Swish model.

\section*{Acknowledgments}
The authors would like to thank Maarten Schoukens at the Eindhoven University of Technology for support on the Bouc-Wen dataset. We would also like to thank Dario Piga and Marco Forgione at the Dalle Molle Institute for Artificial Intelligence in Lugano, Switzerland for constructive feedback.

\bibliographystyle{IEEEtran}
\bibliography{refs}

\appendices

\section{Overall meta-learning process with MAML and system identification data.}
\begin{figure}[!ht]
\centering    
\includegraphics[clip,width=\columnwidth]{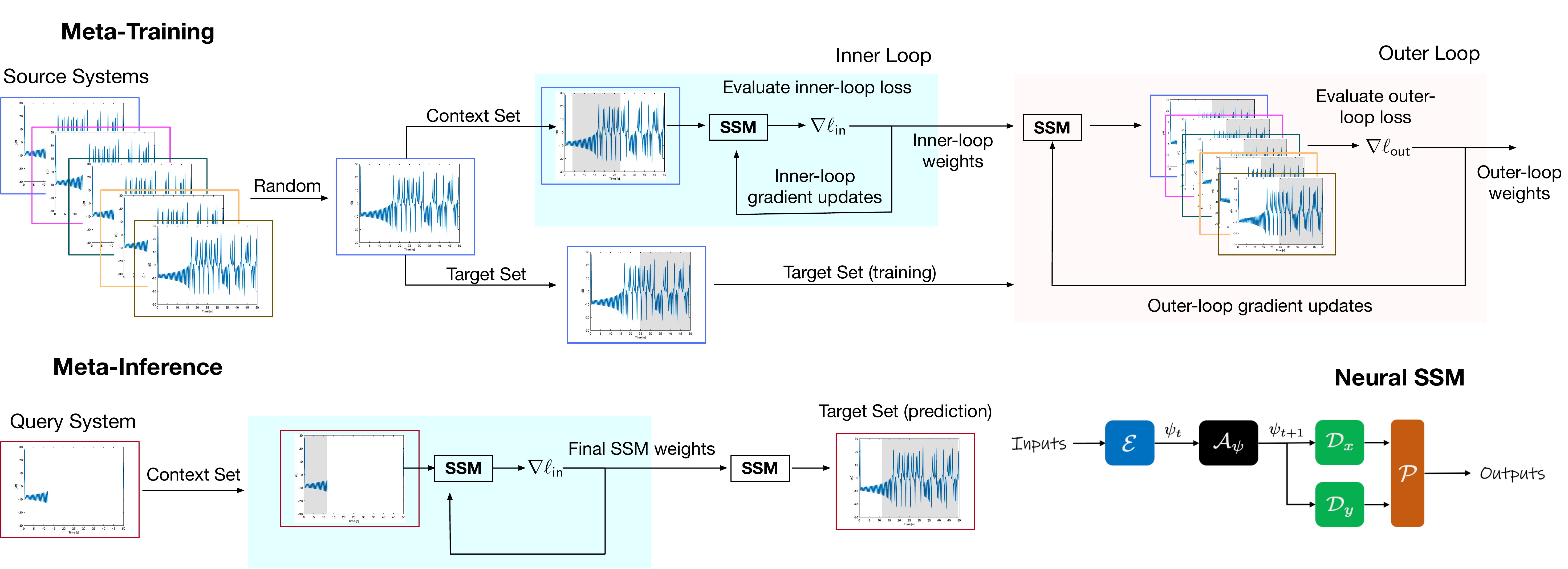}
\caption{Schematic diagram of the overall meta-learning process.}
\label{fig:maml-schema}
\end{figure}

\newpage

\section{Pseudocode for MAML and ANIL.}
\begin{algorithm}[!ht]
\caption{Meta-training SSM with MAML/ANIL}\label{alg:maml}
\begin{algorithmic}[1]
\Require $\omega\leftarrow $ weights of neural SSM
\Require $\Dtrain\leftarrow$ source systems' dataset
\Require $\omega_{\sf in}\subseteq \omega$ \Comment{ANIL: $\omega_{\sf in}\subset \omega$, MAML: $\omega_{\sf in}= \omega$}
\Require $\beta_{\sf in}$, $\beta_{\sf out}$, $M$ \Comment{learning rates and \# iters}
\State Randomly initialize $\omega$
\While{not done} \Comment{\textsc{outer-loop}}
\State Sample batch from $\Dtrain$
\For{$b=1$ to $B$} \Comment{\textsc{inner-loop}}
\State Partition data into $\mathcal{C}^{b}$ and $\mathcal{T}^{b}$
\State $\omega_0^b\leftarrow\omega$ \Comment{copy current weights}
\For{$m=1$ to $M$} \Comment{adaptation steps}
\For{$\omega_l \in \omega_{\sf in}$} \Comment{step through layers}
\State $(\omega_l)^b_m \leftarrow$ update using~\eqref{eq:maml_inner} or~\eqref{eq:anil_inner}
\EndFor
\EndFor
\EndFor
\State $\omega\leftarrow$ update using~\eqref{eq:maml_outer}
\EndWhile
\State Return $\omega_\infty\leftarrow$ final trained weights
\end{algorithmic}
\end{algorithm}

\begin{algorithm}[!ht]
\caption{SSM inference with MAML/ANIL}\label{alg:maml_inference}
\begin{algorithmic}[1]
\Require $\omega_\infty\leftarrow $ weights of meta-trained neural SSM
\Require $\Dtgt\leftarrow$ target system dataset
\Require $\omega_{\sf in}\subseteq \omega$ \Comment{ANIL: $\omega_{\sf in}\subset \omega$, MAML: $\omega_{\sf in}= \omega$}
\Require $\beta_{\sf in}$, $M$ \Comment{learning rates and \# iters} 
\State $\mathcal C^{\star}\leftarrow$ all available data in $\Dtgt$ 
\State $\omega_0^\star\leftarrow\omega_\infty$ \Comment{use meta-trained weights}
\For{$m=1$ to $M$} \Comment{online adaptation}
\For{$\omega_l \in \omega_{\sf in}$} 
\State $(\omega_l)^\star_m \leftarrow$ update using~\eqref{eq:anil_inner}
\EndFor
\EndFor
\State Return $\omega_M^\star\leftarrow$  target NSSM predictive model
\end{algorithmic}
\end{algorithm}

\vfill
\clearpage
\section{Supplementary information for Bouc-Wen Example.}
\begin{figure}[!ht]
\centering
\includegraphics[width=0.9\columnwidth,clip]{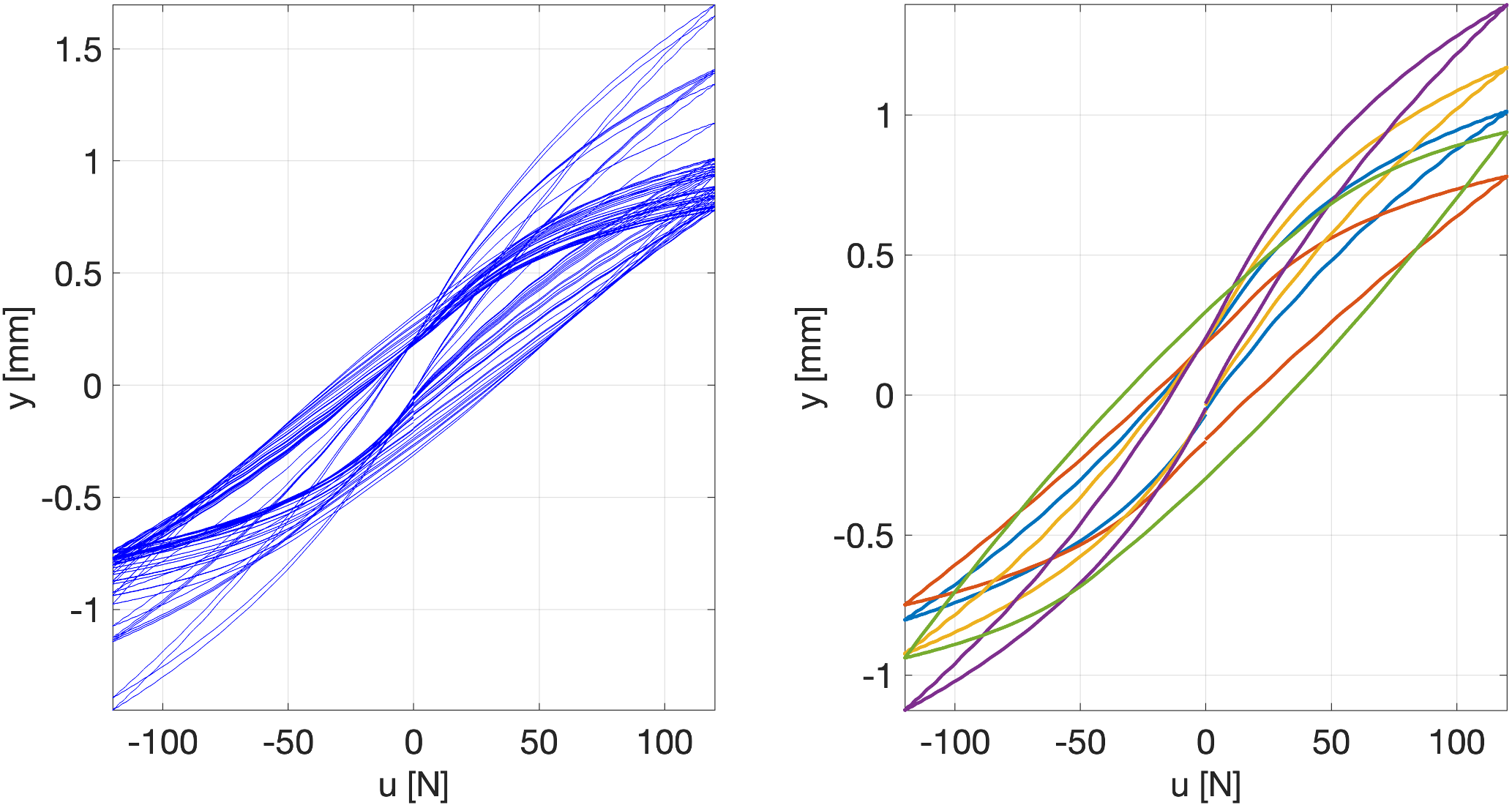}
\caption{(\textit{Left}) Various hysteresis loops over all source systems by varying parameters. (\textit{Right}) Clusters obtained by clustering hysteresis loops.}
\label{fig:bw_hyster}
\end{figure}

\newpage
\section{Supplementary information for Localization Case Study.}
\begin{figure}[!ht]
   \centering    \includegraphics[width=0.79\columnwidth]{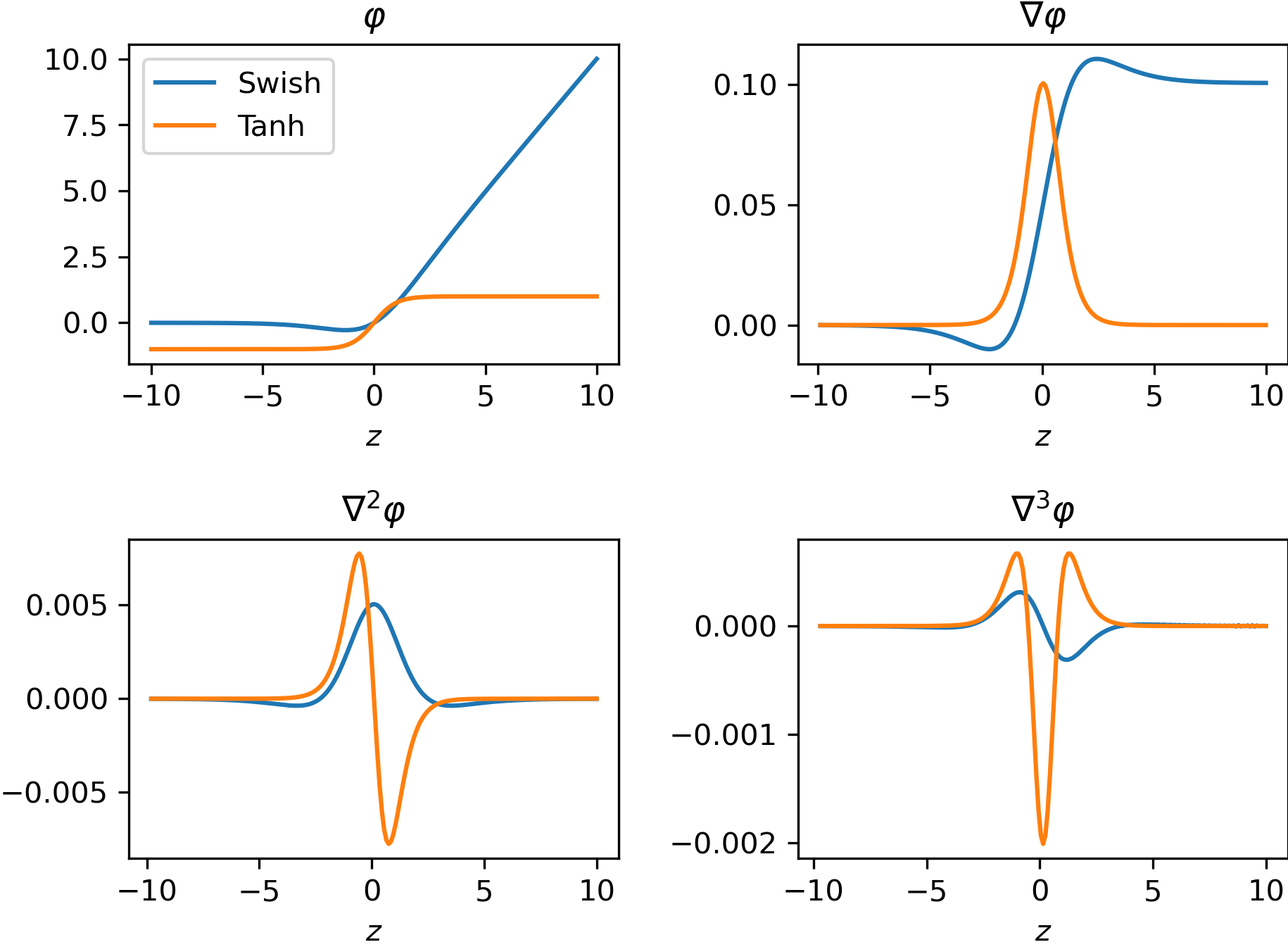}
   \caption{The unbounded \texttt{swish} and bounded \texttt{tanh} activation function and their derivatives.}
\label{fig:swish}
\end{figure}

\begin{figure*}[!htbp]
    \centering
        \begin{subfigure}{0.48\textwidth}
        \centering
        \includegraphics[width=\textwidth, clip=true]{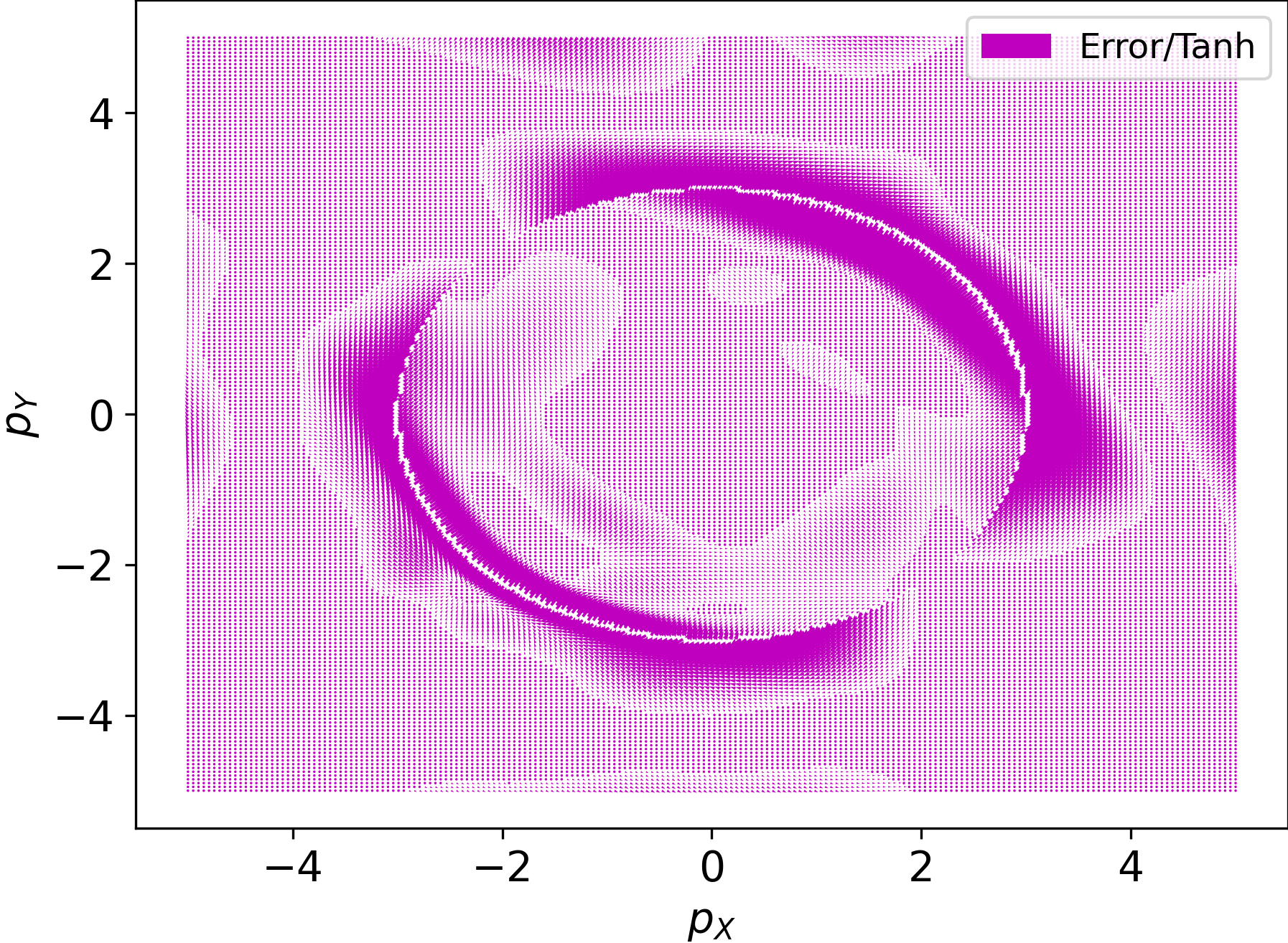}
        \label{fig:image1}
    \end{subfigure}
    \begin{subfigure}{0.48\textwidth}
        \centering
        \includegraphics[width=\textwidth, clip=true]{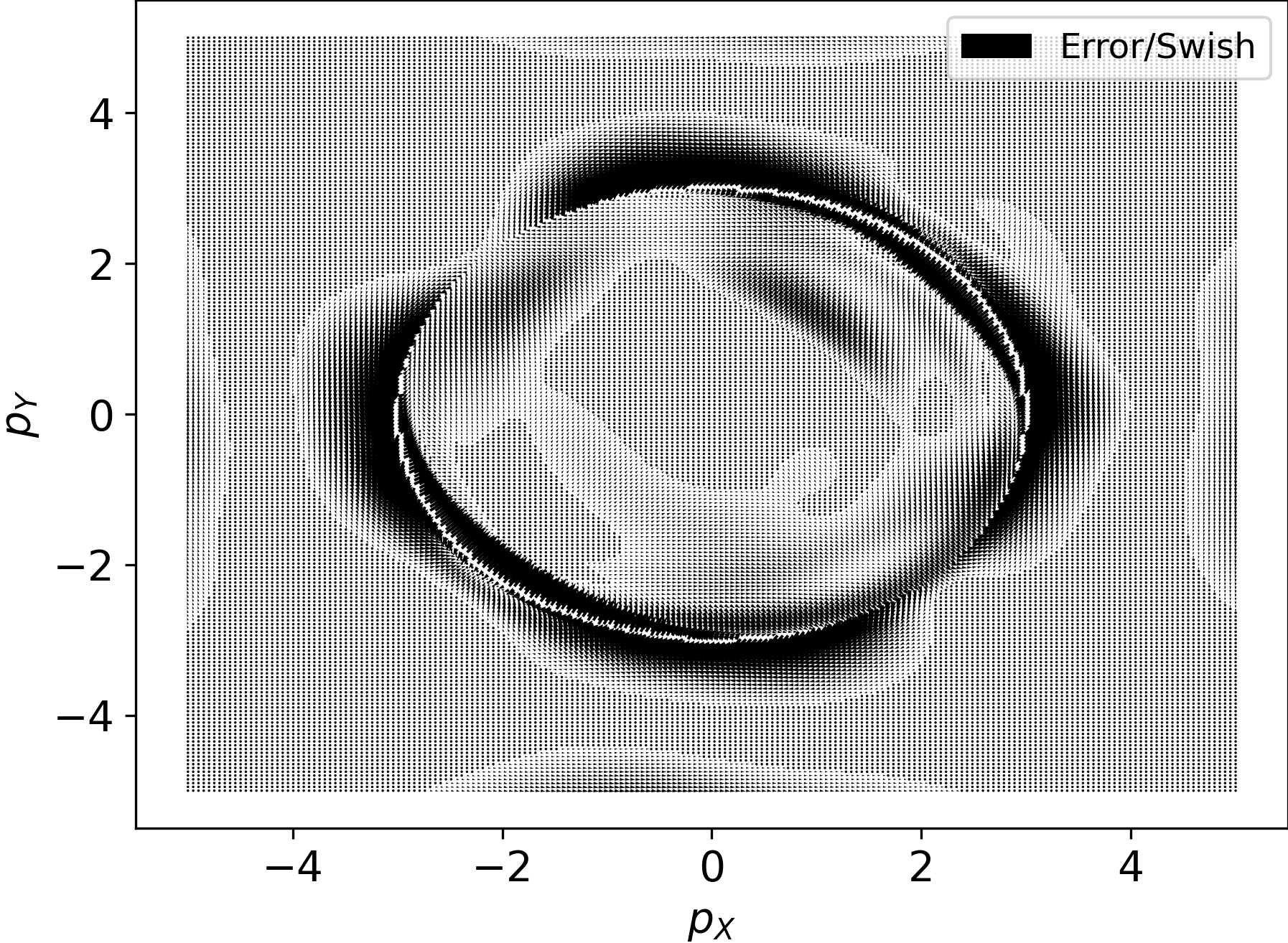}
        \label{fig:image2}
    \end{subfigure}
    
    \caption{Illustration of magnetic field estimation error with MAML with different activation functions.}
    \label{fig:magfield_err}
\end{figure*}

\end{document}